\documentclass[english]{article}
\usepackage[T1]{fontenc}
\usepackage[latin9]{inputenc}
\usepackage{fancyhdr}
\usepackage{algorithm}
\usepackage{algorithmic}
\usepackage{float}
\usepackage{amsthm}
\usepackage{amsmath}
\usepackage{amssymb}
\usepackage[round]{natbib}
\usepackage{mdwlist}
\usepackage{amsfonts}
\usepackage{dsfont}
\usepackage{proceed2e}
\usepackage{babel}

\makeatletter
\floatstyle{ruled}
\newfloat{algorithm}{tbp}{loa}
\floatname{algorithm}{Algorithm}

\numberwithin{equation}{section}
\numberwithin{figure}{section}
\theoremstyle{plain}
\newtheorem{thm}{Theorem}
\theoremstyle{plain}
\newtheorem{lem}[thm]{Lemma}
\newtheorem{cor}[thm]{Corollary}

\theoremstyle{definition}
\newtheorem{Def}{Definition}
\newtheorem{cond}{Condition}

\makeatother

\DeclareMathOperator*{\argmax}{argmax}
\global\long\def\CH{\mathcal{C}}
\global\long\def\history{h}
\global\long\def\mA{\mathcal{A}} \global\long\def\AMO{\mathcal{AMO}}
\global\long\def\heta{\eta_{t-1}}

\newcommand{\Abs}[1]{\left\lvert#1\right\rvert}
\newcommand{\Bracks}[1]{\left[#1\right]}
\newcommand{\Braces}[1]{\left\{#1\right\}}

\newcommand{\BigParens}[1]{\Bigl(#1\Bigr)}
\newcommand{\Et}{{\mathbb{E}}_t}
\newcommand{\Ett}{{\mathbb{E}}_{t'}}
\newcommand{\Vart}{{\mathbb{V}}_t}
\newcommand{\set}[1]{\{#1\}}
\def\Wset{\mathcal{W}}
\def\Zset{\mathcal{Z}}
\def\R{\mathbb{R}}
\def\PE{\text{\mdseries\scshape PolicyElimination}}
\def\DPE{\text{\mdseries\scshape DelayedPE}}
\def\RUCB{\text{\mdseries\scshape RandomizedUCB}}

\newcommand{\E}{\mathop{\mathbb{E}}}
\newcommand{\I}{\mathbb{I}}

\def\pimax{\pi_{\max}} 
\newcommand\vbound[1]{\ensuremath{\bar{V}_{#1}}}
\newcommand\vmax[1]{\ensuremath{\vbound{\max,#1}}}
\newcommand\vepsopt[1]{\ensuremath{\varepsilon_{\operatorname{opt},#1}}}

\def\veps{\varepsilon}
\def\wh{\widehat}
\def\wt{\widetilde}

\DeclareMathOperator{\OPT}{OPT}
\newcommand\Sparse[1]{\mathsf{Sparse}[#1]}

\title{Efficient Optimal Learning for Contextual Bandits}

\author{Miroslav Dudik \\ \texttt{mdudik@yahoo-inc.com} \And Daniel Hsu \\ \texttt{djhsu@rci.rutgers.edu} \And Satyen Kale \\ \texttt{skale@yahoo-inc.com} \And Nikos Karampatziakis \\ \texttt{nk@cs.cornell.edu} \AND John Langford \\ \texttt{jl@yahoo-inc.com} \And Lev Reyzin \\ \texttt{lreyzin@cc.gatech.edu} \And Tong Zhang \\ \texttt{tzhang@stat.rutgers.edu}}

\begin{document}

\maketitle

\begin{abstract}
We address the problem of learning in an online setting where the
learner repeatedly observes features, selects among a set of actions,
and receives reward for the action taken.  We provide the first
efficient algorithm with an optimal regret.  Our algorithm uses a cost
sensitive classification learner as an oracle and has a running time
$\mathrm{polylog}(N)$, where $N$ is the number of classification rules among
which the oracle might choose.  This is exponentially faster than all
previous algorithms that achieve optimal regret in this setting.  Our
formulation also enables us to create an algorithm with regret that is
additive rather than multiplicative in feedback delay as in all
previous work.
\end{abstract}

\section{INTRODUCTION}
The contextual bandit setting consists of the following
loop repeated indefinitely:
\begin{enumerate}
\item The world presents context information as features $x$.
\item The learning algorithm chooses an action $a$ from $K$ possible actions.
\item The world presents a reward $r$ for the action.
\end{enumerate}
The key difference between the contextual bandit setting and standard
supervised learning is that \emph{only }the reward of the chosen action
is revealed. For example, after always choosing the same action several
times in a row, the feedback given provides almost no basis to prefer the
chosen action over another action. In essence, the contextual bandit
setting captures the difficulty of exploration while avoiding the
difficulty of credit assignment as in more general reinforcement learning
settings.

The contextual bandit setting is a half-way point between standard
supervised learning and full-scale reinforcement learning where it
appears possible to construct algorithms with convergence rate
guarantees similar to supervised learning.  Many natural settings
satisfy this half-way point, motivating the investigation of
contextual bandit learning.  For example, the problem of choosing
interesting news articles or ads for users by internet companies can
be naturally modeled as a contextual bandit setting.  In the medical
domain where discrete treatments are tested before approval, the
process of deciding which patients are eligible for a treatment takes
contexts into account. More generally, we can imagine that in a future
with personalized medicine, new treatments are essentially equivalent
to new actions in a contextual bandit setting.

In the i.i.d.\ setting, the world draws a pair $(x,\vec{r})$
consisting of a context and a reward vector from some unknown
distribution $D$, revealing $x$ in Step 1, but only the reward $r(a)$ of the
chosen action $a$ in Step 3. Given a set of policies
$\Pi=\{\pi:X\rightarrow A\}$, the goal is to create an algorithm for
Step 2 which competes with the set of policies.  We measure our
success by comparing the algorithm's cumulative reward to the expected
cumulative reward of the best policy in the set. The difference of the
two is called \emph{regret}.

All existing algorithms for this setting either achieve a suboptimal
regret~\citep{Epoch-Greedy} or require computation linear in the number
of policies~\citep{EXP4,EXP4P}.
In unstructured policy spaces, this computational complexity is the best one
can hope for.
On the other hand, in the case where the rewards of all actions are revealed, the problem is equivalent to cost-sensitive classification, and 
we know of algorithms to efficiently search the space of policies
(classification rules) such as cost-sensitive logistic regression and support vector machines.
In these cases, the space of classification rules is exponential in the number of features, but these problems can be efficiently solved using convex optimization.

Our goal here is to efficiently solve the contextual bandit problems for
similarly large policy spaces.  We do this by reducing the contextual bandit
problem to cost-sensitive classification.  Given a supervised cost-sensitive
learning algorithm as an oracle~\citep{FT}, our algorithm runs in time only
$\mathrm{polylog}(N)$ while achieving regret $O(\sqrt{TK \ln N})$, where $N$ is
the number of possible policies (classification rules), $K$ is the number of actions
(classes), and $T$ is the number of time steps.  This efficiency is achieved in
a modular way, so any future improvement in cost-sensitive learning immediately
applies here.

\subsection{PREVIOUS WORK AND MOTIVATION}

All previous regret-optimal approaches are \emph{measure} based---they work
by updating a measure over policies, an operation which is linear in
the number of policies.  In contrast, regret guarantees scale only
logarithmically in the number of policies.  If not for the
computational bottleneck, these regret guarantees imply that we could
dramatically increase performance in contextual bandit settings using
more expressive policies.  We overcome the computational bottleneck
using an algorithm which works by creating cost-sensitive
classification instances and calling an oracle to choose optimal
policies. Actions are chosen based on the policies returned by the
oracle rather than according to a measure
over all policies. This is reminiscent of AdaBoost~\citep{AdaBoost}, which
creates weighted binary classification instances and calls a ``weak
learner'' oracle to obtain classification rules. These classification
rules are then combined into a final classifier with boosted accuracy.
Similarly as AdaBoost converts a weak learner into a
strong learner, our approach converts a cost-sensitive classification
learner into an algorithm that solves the contextual bandit problem.

In a more difficult version of contextual bandits, an adversary
chooses $(x,\vec{r})$ given knowledge of the learning algorithm (but
not any random numbers).  All known regret-optimal solutions in
the adversarial
setting are variants of the EXP4 algorithm~\citep{EXP4}.  EXP4 achieves
the same regret rate as our algorithm: $O\left(\sqrt{KT\ln N}\right)$,
where $T$ is the number of time steps, $K$ is the number of actions
available in each time step, and $N$ is the number of policies.

Why not use EXP4 in the i.i.d.\ setting?  For example, it is known
that the algorithm can be modified to succeed with high
probability~\citep{EXP4P}, and also for VC classes when the adversary
is constrained to i.i.d.\ sampling.  There are two central benefits
that we hope to realize by directly assuming i.i.d.\ contexts and
reward vectors.
\begin{enumerate}
\item Computational Tractability. Even when the reward vector is fully known,
adversarial regrets scale as $O\left(\sqrt{\ln N}\right)$ while computation scales
as $O(N)$ in general. One attempt to get around this is the follow-the-perturbed-leader
algorithm~\citep{FPL} which provides a computationally tractable solution
in certain special-case structures. This algorithm has no mechanism
for efficient application to arbitrary policy spaces, even given an
efficient cost-sensitive classification oracle. An efficient
cost-sensitive classification oracle has been shown effective in transductive
settings~\citep{Transductive_oracle}.
Aside from the drawback of requiring a transductive setting, the regret
achieved there is substantially worse than for EXP4.
\item Improved Rates. When the world is not completely adversarial, it
  is
possible to achieve substantially lower regrets than are possible
with algorithms optimized for the adversarial setting. For example,
in supervised learning, it is possible to obtain regrets scaling as $O(\log(T))$
with a problem dependent constant~\citep{BHR07}. When the feedback
is delayed by $\tau$ rounds, lower bounds imply that the regret in
the adversarial setting increases by a multiplicative $\sqrt{\tau}$
while in the i.i.d.\ setting, it is possible to achieve an additive regret
of $\tau$~\citep{LSZ09}.
\end{enumerate}
In a direct i.i.d.\ setting, the previous-best approach using a
cost-sensitive classification oracle was given
by $\epsilon$-greedy and epoch greedy algorithms~\citep{Epoch-Greedy}
which have a regret scaling as $O(T^{2/3})$ in the worst case.

There have also been many special-case analyses. For example, theory
of
context-free setting is well understood~\citep{LR,UCB,AE}.
Similarly, good algorithms exist when rewards are linear functions of features~\citep{Linear}
or actions lie in a continuous space with the reward function sampled according to a Gaussian process~\citep{Gaussian}.

\subsection{WHAT WE PROVE}

In Section~\ref{sec:alg} we state the $\PE$ algorithm, and prove the following
regret bound for it.

\newtheorem*{thm:PE}{Theorem \ref{thm:regretbd}}
\begin{thm:PE}
For all distributions $D$ over $(x,\vec{r})$ with $K$ actions, for all
sets of $N$ policies $\Pi$, with probability at least $1-\delta$, the
regret of $\PE$
 (Algorithm~\ref{alg:PE}) over $T$ rounds is at most \[
16\sqrt{2TK\ln\frac{4T^2N}{\delta}}.
\]
\end{thm:PE}

This result can be extended to deal with VC classes, as well as
other special cases.  It forms the simplest method we have of
exhibiting the new analysis.

The new key element of this algorithm is identification of
a distribution over actions which simultaneously achieves small
expected regret and allows estimating value of every policy with small
variance. The existence of
such a distribution is shown \emph{nonconstructively} by a minimax
argument.

$\PE$ is computationally intractable and also
requires exact knowledge of the context distribution (but not the
reward distribution!).  We
show how to address these issues in
Section~\ref{sec:RUCB}
using an algorithm we call $\RUCB$.
Namely, we prove the following theorem.

\newtheorem*{thm:RUCB}{Theorem \ref{thm:rucb-regret}}
\begin{thm:RUCB}
For all distributions $D$ over $(x,\vec{r})$ with $K$ actions, for all
sets of $N$ policies $\Pi$,
with probability at least $1-\delta$, the regret of $\RUCB$
(Algorithm~\ref{alg:RUCB}) over $T$ rounds is at most
\[
O\left(
\sqrt{TK \log\left(TN/\delta\right)}
+ K \log(NK/\delta)
\right).
\]
\end{thm:RUCB}

$\RUCB$'s analysis is substantially more complex,
with a key subroutine being an application of the ellipsoid algorithm
with a cost-sensitive classification oracle
(described in Section~\ref{sec:oracle}). $\RUCB$ does not assume
knowledge of the context distribution, and instead works with the
history of contexts it has observed. Modifying the proof for this
empirical distribution requires a covering
argument over the distributions over policies which uses the
probabilistic method.  The net result is an algorithm with a similar
top-level analysis as $\PE$, but with the running time only
poly-logarithmic in the number of
policies given a cost-sensitive classification oracle.

\newtheorem*{thm:RUCB-Ellipsoid}{Theorem~\ref{thm:ellipsoid}}
\begin{thm:RUCB-Ellipsoid}
In each time step $t$, $\RUCB$ makes at most
$O(\mathrm{poly}(t,K,\log(1/\delta),\log N))$ calls to cost-sensitive
classification oracle, and
requires additional $O(\mathrm{poly}(t,K,\log N))$ processing time.
\end{thm:RUCB-Ellipsoid}

Apart from a tractable algorithm, our analysis can be used to derive
tighter regrets than would be possible in adversarial setting. For
example, in  Section~\ref{sec:delay}, we consider a common setting
where reward feedback is delayed by $\tau$ rounds.  A straightforward
modification of $\PE$ yields a
regret with an \emph{additive} term proportional to $\tau$ compared
with the delay-free setting. Namely, we prove the following.

\newtheorem*{thm:PED}{Theorem \ref{thm:PEdel}}
\begin{thm:PED}
For all distributions $D$ over $(x,\vec{r})$ with $K$ actions, for all
sets of $N$ policies $\Pi$, and all delay intervals $\tau$,
with probability at least $1-\delta$,
the regret of
$\DPE$ (Algorithm~\ref{alg:PEdelay})
is at most
 \[
16\sqrt{2K\ln\frac{4T^2N}{\delta}}\left(\tau+\sqrt{T}\right).
\]
\end{thm:PED}

We start next with precise settings and definitions.

\section{SETTING AND DEFINITIONS}\label{sec:SnD}

\subsection{THE SETTING}

Let $A$ be the set of $K$ actions, let $X$ be the domain of contexts $x$, and
let $D$ be an arbitrary joint distribution on $(x,\vec{r})$.
We denote the marginal distribution of $D$ over $X$ by $D_X$.

We denote $\Pi$ to be a finite set of policies $\{\pi : X \rightarrow
A  \}$, where each policy
$\pi$, given a context $x_t$ in round $t$, chooses the action
$\pi(x_t)$. The cardinality of
$\Pi$ is denoted by $N$. Let $\vec{r}_t\in [0,1]^K$ be the vector of
rewards, where $r_t(a)$ is the reward of action $a$ on round $t$.

In the i.i.d.\ setting, on each round $t = 1 \ldots T$, the world chooses
$(x_t,\vec{r}_t)$ i.i.d.\ according to $D$
and reveals $x_t$ to the learner. The learner, having access to
$\Pi$,
chooses action $a_t\in \{1,\ldots,K\}$.
Then the world reveals reward $r_t(a_t)$ (which we call $r_t$ for short)
to the learner, and the interaction proceeds to the next round.



We consider two modes of accessing the set of policies $\Pi$. The
first option is through the enumeration of all policies. This is
impractical in general, but suffices for the illustrative purpose of
our first algorithm. The second option is an oracle access, through an
\emph{argmax oracle}, corresponding to a cost-sensitive
learner:

\begin{Def}
For a set of policies $\Pi$, an argmax oracle ($\AMO$ for short), is
an algorithm, which for any sequence
$\set{(x_{t'},\vec{r}_{t'})}_{{t'}=1\dotsc t}$, $x_{t'}\in X$, $\vec{r}_{t'}\in\R^K$, computes
\[
   \arg\max_{\pi \in \Pi} \sum_{{t'}=1\dotsc t} r_{t'}(\pi(x_{t'}))
\enspace.
\]
\end{Def}

The reason why the above can be viewed as a cost-sensitive
classification oracle is that
vectors of rewards $\vec{r}_{t'}$ can be interpreted as negative costs
and hence the policy returned by $\AMO$ is the optimal cost-sensitive
classifier on the given data.

\subsection{EXPECTED AND EMPIRICAL REWARDS}

Let the expected instantaneous \emph{reward} of a policy $\pi \in \Pi$ be denoted
by
\[
\eta_D(\pi) \doteq \E_{(x,\vec{r})\sim D} [r(\pi(x))]
\enspace.
\]
The best policy $\pimax \in \Pi$ is that which maximizes $\eta_D(\pi)$.
More formally,
$$\pimax \doteq \argmax_{\pi \in \Pi}{\eta_D(\pi)}
\enspace.
$$

We define $h_t$ to be the \emph{history} at time $t$ that the learner has seen.  Specifically
$$h_t = \bigcup_{t' =1\ldots t}
(x_{t'},a_{t'},r_{t'},p_{t'})
\enspace,
$$ where $p_{t'}$ is the
probability of the algorithm choosing action $a_{t'}$ at time $t'$.
Note that $a_{t'}$ and $p_{t'}$ are produced by the learner while $x_{t'},
r_{t'}$ are produced by nature.  We write $x \sim h$ to
denote choosing $x$ uniformly at random from the $x$'s in history $h$.

Using the history of past actions and probabilities with which they were
taken, we can form an unbiased estimate of the policy value for any $\pi\in\Pi$:
$$\eta_{t}(\pi) \doteq \frac{1}{t}\sum_{(x,a,r,p)\in
  h_t}\frac{r\I(\pi(x)=a)}{p}.$$
The unbiasedness follows, because
$\E_{a \sim p} \frac{r\I(\pi(x)=a)}{p(a)} = \sum_a
p(a)\frac{r\I(\pi(x)=a)}{p(a)} = r(\pi(x)) $.
The empirically best policy at time $t$ is denoted
\[
\pi_t \doteq \argmax_{\pi \in \Pi}{\eta_{t}(\pi)}.
\]

\subsection{REGRET}

The goal of this work is to obtain a learner that has small
\emph{regret} relative to the expected performance of
$\pimax$ over $T$ rounds, which is
\begin{equation}\label{def:regret}
\sum_{t=1\dotsc T} \left( \eta_D(\pimax) - r_t \right).
\end{equation}
We say that the regret of the learner over $T$ rounds is bounded by
$\epsilon$ with probability at least $1-\delta$, if
\[
\Pr\left[
\sum_{t=1\dotsc T} \left( \eta_D(\pimax) - r_t \right)
\leq \epsilon
\right] \geq 1-\delta
\]
where the probability is taken with respect to the random pairs
$(x_t,\vec{r}_t) \sim D$ for $t=1\dotsc T$, as well as any internal
randomness used by the learner.

We can also define notions of regret and empirical regret for policies $\pi$. For all $\pi\in\Pi$, let
\begin{align*}
\Delta_D(\pi) &=  \eta_D(\pimax) - \eta_D(\pi)
\enspace,
\\
\Delta_t(\pi) &= \eta_t(\pi_t) - \eta_t(\pi)
\enspace.
\end{align*}



Our algorithms work by choosing distributions over policies, which in turn
then induce distributions over actions.
For any distribution $P$ over policies $\Pi$, let
$W_{P}(x,a)$ denote the induced conditional distribution over actions $a$
given the context $x$:
\begin{equation} \label{eq:induced}
W_{P}(x,a) \doteq \sum_{\pi \in \Pi : \pi(x) = a} P(\pi)
\enspace.
\end{equation}
In general, we shall use $W$, $W'$ and $Z$ as conditional
probability distributions over the actions $A$ given contexts $X$,
i.e., $W:X\times A\to[0,1]$ such that $W(x,\cdot)$ is a probability
distribution over $A$ (and similarly for $W'$ and $Z$).
We shall think of $W'$ as a smoothed version of
$W$ with a minimum action probability of $\mu$ (to be defined by the algorithm), such that
$$
W'(x,a) = (1-K\mu)W(x,a) + \mu
\enspace.
$$
Conditional distributions such as $W$ (and $W'$, $Z$, etc.) correspond
to randomized policies. We define notions true and empirical
value and regret for them as follows:
\begin{align*}
\eta_D(W) &\doteq  \E_{(x,\vec{r})\sim D} [\vec{r}\cdot W(x)]
\\
\eta_t(W) &\doteq \frac{1}{t}\sum_{(x, a, r, p) \in h_t} \frac{r W(x,
  a)}{p}
\\
\Delta_D(W)&\doteq  \eta_D(\pimax) - \eta_D(W)
\\
\Delta_t(W)&\doteq \eta_t(\pi_t) - \eta_t(W)
\enspace.
\end{align*}

\section{POLICY ELIMINATION} \label{sec:alg}

The basic ideas behind our approach are demonstrated in our first algorithm:
$\PE$ (Algorithm~\ref{alg:PE}).

\begin{algorithm}
\caption{$\PE$($\Pi$,$\delta$,$K$,$D_X$)}\label{alg:PE}

Let $\Pi_{0}=\Pi$ and
history $h_{0}=\emptyset$
\\
Define: $\delta_t \doteq \delta\,/\,4Nt^2$
\\
Define: $\displaystyle b_{t} \doteq 2\sqrt{\frac{2K\ln(1/\delta_t)}{t}}$
\\
Define: $\displaystyle \mu_t \doteq \min\Braces{\frac{1}{2K}\,,\,\sqrt{\frac{\ln(1/\delta_t)}{2Kt}}}$
\\
For each timestep $t= 1 \ldots T$, observe $x_t$ and do:
\begin{enumerate}
\item Choose distribution $P_t$ over $\Pi_{t-1}$ s.t.\
          $\forall\ \pi\in\Pi_{t-1}$:
\label{a:stepPC}
          \[
          \E_{x\sim D_X}\left[\frac{1}{(1-K\mu_t)W_{P_t}(x,\pi(x)) +
              \mu_t}\right] \le 2K
          \]
\item \label{a:stepPJ} Let $W'_t(a)=(1-K\mu_t)W_{P_t}(x_t,a) + \mu_t$ for all
         $a\in A$
\item Choose $a_t\sim W'_t$
\item Observe reward $r_t$
\item Let $\displaystyle
    \Pi_{t}=\Bigl\{
    \pi\in\Pi_{t-1}:$
\\\hphantom{Let $\Pi_{t}=\Bigl\{$}
   $\displaystyle
    \eta_{t}(\pi)\ge\BigParens{\max_{\pi'\in
    \Pi_{t-1}}\!\!\!\eta_{t}(\pi')} - 2b_t\Bigr\}$
    \label{a:stepElim}
\item Let $h_{t}=h_{t-1} \cup (x_t,a_t,r_t,W'_t(a_t))$
\end{enumerate}

\end{algorithm}

The key step is
Step~\ref{a:stepPC}, which finds a distribution over
policies which induces low variance in the estimate of the value of
all policies. Below we use minimax theorem to show that such
a distribution always exists. How to find this distribution is not specified here,
but in Section~\ref{sec:oracle} we develop a method based on
the ellipsoid algorithm. Step~\ref{a:stepPJ} then projects this distribution onto a
distribution over actions and applies smoothing.
Finally, Step~\ref{a:stepElim} eliminates the policies that have been determined to be suboptimal
(with high probability). 


\subsection*{ALGORITHM ANALYSIS}\label{sec:analysis}

We analyze $\PE$ in several steps. First, we prove the existence of $P_t$
in Step~\ref{a:stepPC}, provided that $\Pi_{t-1}$ is non-empty. We
recast the feasibility problem in Step~\ref{a:stepPC} as a game
between two players: Prover, who is trying to produce $P_t$, and
Falsifier, who is trying to find $\pi$ violating the
constraints. We give more power to Falsifier and allow him to choose a
distribution over $\pi$ (i.e., a randomized policy) which would
violate the constraints.

Note that any policy $\pi$ corresponds to a
point in the space of randomized policies (viewed as functions $X\times A\to[0,1]$), with
$\pi(x,a)\doteq\I(\pi(x)=a)$. For any distribution $P$ over policies
in $\Pi_{t-1}$, the induced randomized policy $W_P$ then corresponds to a
point in the convex hull of $\Pi_{t-1}$. Denoting the convex hull of
$\Pi_{t-1}$ by $\CH$, Prover's choice by $W$ and
Falsifier's choice by $Z$, the feasibility of
Step~\ref{a:stepPC} follows by the following lemma:

\begin{lem} \label{lem:minimax}
Let $\CH$ be a compact and convex set of randomized policies.
Let $\mu\in(0,1/K]$ and for any $W\in\CH$,
$W'(x,a)\doteq(1-K\mu)W(x,a) +
              \mu$. Then for all
distributions $D$,
\[
\min_{W\in\CH} \max_{Z\in\CH}
\E_{x \sim D_X}\E_{a\sim Z(x,\cdot)} \left[ \frac{1}{W'(x,a)}
              \right]
\leq \frac{K}{1-K\mu}
\enspace.
\]
\end{lem}

\begin{proof}
Let $f(W,Z)\doteq \E_{x \sim D_X}\E_{a\sim Z(x,\cdot)} [1/W'(x,a)]$ denote the
inner expression of the minimax problem.
Note that
$f(W,Z)$ is:
\begin{itemize}
\item\emph{everywhere defined}: Since $W'(x,a)\ge \mu$, we obtain that
  $1/W'(x,a)\in[0,1/\mu]$, hence the expectations are defined for all
  $W$ and $Z$.
\item\emph{linear in $Z$}: Linearity follows from rewriting $f(W,Z)$ as
\[
   f(W,Z) = \E_{x \sim D_X} \sum_{a\in A} \left[ \frac{Z(x,a)}{W'(x,a)} \right].
\]
\item\emph{convex in $W$}: Note that $1/W'(x,a)$ is convex in $W(x,a)$
  by convexity of $1/(c_1 w+ c_2)$ in $w\ge0$, for $c_1\ge0$, $c_2>0$. Convexity
  of $f(W,Z)$ in $W$ then
  follows by taking expectations over $x$ and $a$.
\end{itemize}
Hence, by Theorem~\ref{thm:sion} (in Appendix~\ref{sec:sion}),
min and max can be reversed without affecting the value:
\[
   \min_{W\in\CH} \max_{Z\in\CH} f(W,Z)
 =
   \max_{Z\in\CH} \min_{W\in\CH} f(W,Z)
\enspace.
\]
The right-hand side can be further upper-bounded by $\max_{Z\in\CH}
f(Z,Z)$,
which is upper-bounded by
\begin{align}
\notag
&f(Z,Z)
  =
  \E_{x \sim D_X} \sum_{a\in A} \left[ \frac{Z(x,a)}{Z'(x,a)} \right]
\\
\tag*{\qed}
&\quad{}\le
 \E_{x \sim D_X}
 \!\!\!\!\!\sum_{\substack{a\in A:\\Z(x,a)>0}}\!\!\!\!\!
 \left[ \frac{Z(x,a)}{(1-K\mu)Z(x,a)}
             \right]
 =
  \frac{K}{1-K\mu}
\enspace.\!
\end{align}
\renewcommand{\qed}{}
\end{proof}

\begin{cor}
\label{cor:minimax}
The set of distributions satisfying constraints of Step \ref{a:stepPC}
is non-empty.
\end{cor}

Given the existence of 
$P_t$, we will see below that
the constraints in Step \ref{a:stepPC} ensure low variance of the
policy value estimator $\eta_t(\pi)$ for all $\pi\in\Pi_{t-1}$. The
small variance is used to ensure accuracy of policy
elimination in Step \ref{a:stepElim} as quantified in the following
lemma:

\begin{lem}
With probability at least $1-\delta$, for all $t$:
\begin{enumerate}
\item $\pimax\in\Pi_t$ (i.e., $\Pi_t$ is non-empty)
\item $\eta_D(\pimax)-\eta_D(\pi)\le 4b_t$ for all $\pi\in\Pi_t$
\end{enumerate}
\end{lem}
\begin{proof}
We will show that
for any policy $\pi\in\Pi_{t-1}$, the probability that $\eta_{t}(\pi)$
deviates from  $\eta_D(\pi)$ by more that $b_t$ is at most
$2\delta_t$. Taking the union bound over all policies and all time
steps we find that with probability at least $1-\delta$,
\begin{equation}
\label{eq:dev}
 \Abs{\eta_{t}(\pi)-\eta_D(\pi)}\le b_t
\end{equation}
for all $t$ and all $\pi\in\Pi_{t-1}$. Then:
\begin{enumerate}
\item By the triangle inequality, in each time step, $\eta_t(\pi) \le
  \eta_t(\pimax) + 2b_t$ for all $\pi\in\Pi_{t-1}$, yielding the first
  part of the lemma.
\item Also by the triangle inequality, if
  $\eta_D(\pi)<\eta_D(\pimax)-4b_t$ for $\pi\in\Pi_{t-1}$, then
  $\eta_t(\pi)<\eta_t(\pimax)-2b_t$. Hence the policy $\pi$ is
  eliminated
  in Step~\ref{a:stepElim}, yielding the second part of the lemma.
\end{enumerate}

It remains to show Eq.~\eqref{eq:dev}. We fix the policy $\pi\in\Pi$
and
time $t$, and show that the deviation bound is violated
with probability at most $2\delta_t$. Our argument rests
on Freedman's inequality
(see Theorem~\ref{thm:freedmanvar} in Appendix~\ref{sec:concentration}).
Let
\[
    y_t = \frac{r_t \I(\pi(x_t)=a_t)}{W'_t(a_t)}
\enspace,
\]
i.e., $\eta_{t}(\pi)=(\sum_{t'=1}^t y_{t'})/t$. Let $\Et$ denote the
conditional expectation $\E[{}\cdot{}\vert\, h_{t-1}]$. To use
Freedman's inequality, we need to bound the range of $y_t$ and its
conditional second moment $\Et[y_t^2]$.

Since $r_t\in[0,1]$ and $W'_t(a_t)\ge\mu_t$, we have the bound
\[
   0\le y_t \le 1/\mu_t\doteq R_t
\enspace.
\]
Next,
\begin{align}
\notag
   \Et[y_t^2]
&=
   \E_{(x_t,\vec{r}_t)\sim D}
   \E_{a_t\sim W'_t}
   \Bracks{y_t^2}
\\
\notag
&=
   \E_{(x_t,\vec{r}_t)\sim D} \E_{a_t\sim W'_t}
   \Bracks{
   \frac{r_t^2\I(\pi(x_t)=a_t)}
          {W'_t(a_t)^2}}
\\
\label{eq:var:rBounded}
&\le
   \E_{(x_t,\vec{r}_t)\sim D}
   \Bracks{
   \frac{W'_t(\pi(x_t))}
          {W'_t(\pi(x_t))^2}}
\\
\label{eq:var:constr}
&=
   \E_{x_t\sim D}
   \Bracks{
   \frac{1}
          {W'_t(\pi(x_t))}}
   \le 2K
\enspace.
\end{align}
where Eq.~\eqref{eq:var:rBounded} follows by boundedness of $r_t$ and
Eq.~\eqref{eq:var:constr} follows from the constraints in
Step~\ref{a:stepPC}. Hence,
\[
   \sum_{t'=1\dotsc t} \Ett[y_{t'}^2] \le 2Kt \doteq V_t
\enspace.
\]

Since $(\ln t)/t$ is decreasing for $t\ge 3$, we obtain
that $\mu_t$ is non-increasing (by separately analyzing $t=1$, $t=2$,
$t\ge 3$). Let $t_0$ be the first $t$ such that
$\mu_t<1/2K$. Note that $b_t\ge 4K\mu_t$, so for $t<t_0$, 
we have $b_t\ge 2$
and $\Pi_t=\Pi$. Hence, the deviation bound holds for $t<t_0$.

Let $t\ge t_0$. For $t'\le t$, by the monotonicity of $\mu_t$
\[
   R_{t'} = 1/\mu_{t'} \le 1/\mu_{t} = \sqrt{\frac{2Kt}{\ln(1/\delta_t)}} =
   \sqrt{\frac{V_t}{\ln(1/\delta_t)}}
\enspace.
\]
Hence, the assumptions of Theorem~\ref{thm:freedmanvar} are satisfied, and
\[
    \Pr\Bracks{\Abs{\eta_t(\pi)-\eta_D(\pi)}\ge b_t} \le 2\delta_t
\enspace.
\]
The union bound over $\pi$ and $t$ yields Eq.~\eqref{eq:dev}.
\end{proof}

This immediately implies that the cumulative regret
is bounded by
\begin{eqnarray}
\sum_{t=1\dotsc T}\!\!\left( \eta_D(\pimax) - r_t \right) & \le &
8\sqrt{2K\ln\frac{4NT^2}{\delta}} \sum_{t=1}^{T}\frac{1}{\sqrt{t}}\nonumber\\
&\leq&16\sqrt{2TK\ln\frac{4T^2N}{\delta}}\label{eq:sumrtt}\end{eqnarray}
and gives us the following theorem.

\begin{thm}\label{thm:regretbd}
For all distributions $D$ over $(x,\vec{r})$ with $K$ actions, for all
sets of $N$ policies $\Pi$, with probability at least $1-\delta$, the
regret of $\PE$ (Algorithm~\ref{alg:PE}) over $T$ rounds is at most \[
16\sqrt{2TK\ln\frac{4T^2N}{\delta}}
\enspace.
\]
\end{thm}

\section{THE RANDOMIZED UCB ALGORITHM}\label{sec:RUCB}

\begin{algorithm}
\caption{$\RUCB$($\Pi$,$\delta$,$K$)} \label{alg:RUCB}

Let $h_{0} \doteq \emptyset$ be the initial history.

Define the following quantities:
\[
C_t \doteq 2\log\left(\frac{Nt}{\delta}\right)
\quad \text{and} \quad
\mu_t \doteq \min\left\{ \frac{1}{2K}, \ \sqrt{\frac{C_t}{2Kt}} \right\}
.
\]

For each timestep $t= 1 \ldots T$, observe $x_t$ and do:
\begin{enumerate}

\item
Let $P_t$ be a distribution over $\Pi$ that approximately solves the
optimization problem
\begin{equation} \label{eq:rucb-opt}
\begin{aligned}
& \min_P \sum_{\pi\in\Pi} P(\pi) \Delta_{t-1}(\pi) \\
& \text{s.t.} \quad \text{for all distributions $Q$ over $\Pi$}: \\
& \E_{\pi\sim Q} \left[ \frac1{t-1} \sum_{i=1}^{t-1}
\frac{1}{(1-K\mu_t) W_P(x_i,\pi(x_i)) + \mu_t} \right]
\\ & \qquad
\leq \max\left\{ 4K,\ \frac{(t-1) \Delta_{t-1}(W_Q)^2}{180 C_{t-1}}
\right\}
\end{aligned}
\end{equation}
so that the objective value at $P_t$ is within $\vepsopt{t} = O(\sqrt{K
C_t/t})$ of the optimal value, and so that each constraint is satisfied
with slack $\leq K$.

\item Let $W_t'$ be the distribution over $A$ given by
\[ W_t'(a) \doteq (1-K\mu_t) W_{P_t}(x_t,a) + \mu_t \]
for all $a \in A$.

\item Choose $a_t \sim W_t'$.

\item Observe reward $r_t$.

\item Let $h_{t} \doteq h_{t-1} \cup (x_t,a_t,r_t,W_t'(a_t))$.

\end{enumerate}
\end{algorithm}

$\PE$ is the simplest exhibition of the minimax argument, but
it has some drawbacks:
\begin{enumerate}
\item The algorithm keeps explicit track of the space of good policies
(like a version space), which is difficult to implement efficiently in
general.

\item If the optimal policy is mistakenly eliminated by chance, the
algorithm can never recover.

\item The algorithm requires perfect knowledge of the distribution $D_X$
over contexts.

\end{enumerate}

These difficulties are addressed by $\RUCB$ (or RUCB for short), an algorithm
which we present and analyze in this section. Our approach is
reminiscent of the UCB
algorithm~\citep{UCB}, developed for context-free setting, which keeps an
upper-confidence bound on the expected reward for each action.
However, instead of
choosing the highest upper confidence bound, we randomize over choices
according to the value of their empirical performance.
The algorithm has the following properties:
\begin{enumerate}
\item The optimization step required by the algorithm always considers the
  full set of policies (i.e., explicit tracking of the set of
  good policies is avoided), and thus it can be efficiently
  implemented using an argmax oracle.  We discuss this further in
  Section~\ref{sec:oracle}.

\item Suboptimal policies are implicitly used with decreasing frequency by
using a non-uniform variance constraint that depends on a policy's
estimated regret.
A consequence of this is a bound on the value of the optimization, stated
in Lemma~\ref{lem:value-again} below.

\item Instead of $D_X$, the algorithm uses the history of previously
  seen contexts. The effect of this approximation is quantified in
  Theorem~\ref{thm:unlabeled} below.

\end{enumerate}

The regret of $\RUCB$ is the following:
\begin{thm} \label{thm:rucb-regret}
For all distributions $D$ over $(x,\vec{r})$ with $K$ actions, for all
sets of $N$ policies $\Pi$,
with probability at least $1-\delta$, the regret of $\RUCB$
(Algorithm~\ref{alg:RUCB}) over $T$ rounds is at most
\[
O\left(
\sqrt{TK \log\left(TN/\delta\right)}
+ K \log(NK/\delta)
\right).
\]
\end{thm}
The proof is given in Appendix~\ref{sec:rucb-appendix:regret}%
.
Here, we present an overview of the analysis.

\subsection{EMPIRICAL VARIANCE ESTIMATES} \label{sec:unlabeled}

A key technical prerequisite for the regret analysis is the accuracy
of the empirical variance estimates.
For a distribution $P$ over policies $\Pi$ and a particular policy $\pi \in
\Pi$, define
\begin{align*}
V_{P,\pi,t} & = \E_{x\sim D_X}\left[ \frac1{(1-K\mu_t) W_P(x,\pi(x)) +
\mu_t} \right] \\
\wh{V}_{P,\pi,t} & = \frac1{t-1} \sum_{i=1}^{t-1} \frac1{(1-K\mu_t)
W_P(x_i,\pi(x_i)) + \mu_t}
.
\end{align*}
The first quantity $V_{P,\pi,t}$ is (a bound on) the variance incurred by
an importance-weighted estimate of reward in round $t$ using the action
distribution induced by $P$, and the second quantity $\wh{V}_{P,\pi,t}$ is an
empirical estimate of $V_{P,\pi,t}$ using the finite sample
$\{x_1,\dotsc,x_{t-1}\} \subseteq X$ drawn from $D_X$.
We show that for all distributions $P$ and all $\pi \in \Pi$,
$\wh{V}_{P,\pi,t}$ is close to $V_{P,\pi,t}$ with high probability.

\begin{thm} \label{thm:unlabeled}
For any $\epsilon \in (0,1)$, with probability at least $1-\delta$,
\[
V_{P,\pi,t}
\leq (1+\epsilon) \cdot \wh{V}_{P,\pi,t} + \frac{7500}{\epsilon^3} \cdot K
\]
for all distributions $P$ over $\Pi$, all $\pi \in \Pi$, and all $t \geq
16K\log(8KN/\delta)$.
\end{thm}
The proof appears in Appendix~\ref{sec:finite-appendix}%
.

\subsection{REGRET ANALYSIS}

Central to the analysis is the following lemma that bounds the value of the
optimization in each round. It is a direct
corollary of Lemma~\ref{lem:value} in
Appendix~\ref{sec:rucb-appendix:regret}%
.
\begin{lem} \label{lem:value-again}
If $\OPT_t$ is the value of the optimization problem~\eqref{eq:rucb-opt} in
round $t$, then
\[
\OPT_t
\ \leq\ O\left(\sqrt{\frac{K C_{t-1}}{t-1}}\right)\ =\ O\left(\sqrt{\frac{K \log(Nt/\delta)}{t}}\right).
\]
\end{lem}
This lemma implies that the algorithm is always able to select a
distribution over the policies that focuses mostly on the policies with low
estimated regret.
Moreover, the variance constraints ensure that good policies never appear
too bad, and that only bad policies are allowed to incur high variance in
their reward estimates.
Hence, minimizing the objective in~\eqref{eq:rucb-opt} is an effective
surrogate for minimizing regret.

The bulk of the analysis consists of analyzing the variance of the
importance-weighted reward estimates $\eta_t(\pi)$, and showing how they
relate to their actual expected rewards $\eta_D(\pi)$.
The details are deferred to Appendix~\ref{sec:rucb-appendix}%
.

%

\section{USING AN ARGMAX ORACLE}
\def\unlabeled{\mathcal{X}}

\label{sec:oracle}

In this section, we show how to solve the optimization
problem~\eqref{eq:rucb-opt} using the argmax oracle ($\AMO$) for our set of policies.
Namely, we describe an algorithm running in polynomial time {\em independent}\footnote{Or rather dependent only on $\log N$, the representation size of a policy.} of the number of policies, which makes queries to $\AMO$ to compute a distribution over policies suitable for the optimization step of Algorithm~\ref{alg:RUCB}.

This algorithm relies on the ellipsoid method. The ellipsoid method is a general technique for solving convex programs equipped with a \emph{separation oracle}. A separation oracle is defined as follows:
\begin{Def}
Let $S$ be a convex set in $\mathbb{R}^n$. A separation oracle for $S$ is an algorithm that, given a point $x \in \mathbb{R}^n$, either declares correctly that $x \in S$, or produces a hyperplane $H$ such that $x$ and $S$ are on opposite sides of $H$.
\end{Def}
We do not describe the ellipsoid algorithm here (since it is standard), but only spell out
its key properties in the following lemma. For a point $x \in \mathbb{R}^n$ and $r \geq 0$, we use the notation $B(x, r)$ to denote the $\ell_2$ ball of radius $r$ centered at $x$.
\begin{lem} \label{lem:ellipsoid-description}
Suppose we are required to decide whether a convex set $S \subseteq
\mathbb{R}^n$ is empty or not. We are given a separation oracle for
$S$ and two numbers $R$ and $r$, such that $S \in B(0, R)$ and if $S$
is non-empty, then there is a point $x^\star$ such that $S \supseteq
B(x^\star, r)$. The ellipsoid algorithm decides correctly if $S$ is
empty or not, by executing at most $O(n^2\log(\frac{R}{r}))$
iterations, each involving one call to the separation oracle and
additional $O(n^2)$ processing time.
\end{lem}

We now write a convex program whose solution is the required distribution, and show how to solve it using the ellipsoid method by giving a separation oracle for its feasible set using $\AMO$.

Fix a time period $t$. Let $\unlabeled_{t-1}$ be the set of all contexts seen so far, i.e. $\unlabeled_{t-1} = \{x_1, x_2, \ldots, x_{t-1}\}$.  We embed all policies $\pi \in \Pi$ in $\mathbb{R}^{(t-1)K}$, with coordinates identified with
$(x,a)\in \unlabeled_{t-1}\times A$. With abuse of notation, a policy $\pi$ is
represented by the vector $\pi$ with coordinate $\pi(x,a)=1$ if $\pi(x)=a$ and
$0$ otherwise. Let $\CH$ be the convex hull of all policy vectors $\pi$. Recall that a distribution $P$ over policies corresponds to a point inside $\CH$, i.e.,
$W_P(x,a) = \sum_{\pi:\pi(x)=a} P(\pi)$, and that $W'(x,a)=(1- \mu_t K)W(x,a)+\mu_t$, where $\mu_t$ is as defined in Algorithm~\ref{alg:RUCB}.
Also define $\beta_t = \frac{t-1}{180C_{t-1}}$. In the following, we use the notation $x \sim \history_{t-1}$ to denote a context drawn uniformly at random from $\unlabeled_{t-1}$.

Consider the following convex program:
\begin{align}
&\min\  s \text{ s.t.} \notag\\
&\Delta_{t-1}(W)\ \leq\ s \label{eq:low-regret}\\
&W\ \in\ \CH\label{eq:ch}\\
&\forall Z \in \CH: \notag \\
&\E_{x \sim \history_{t-1}}\!\!
   \left[\sum_a \frac{Z(x, a)}{W'(x,a)}\right]\!\leq \max\{4K,
   \beta_t\Delta_{t-1}(Z)^2\}\!\!
\label{eq:var-bound}
\end{align}
We claim that this program is equivalent to the RUCB optimization
problem~(\ref{eq:rucb-opt}), up to finding an explicit distribution
over policies which corresponds to the optimal solution. This can be
seen as follows. Since we require $W \in \CH$, it can be interpreted
as being equal to $W_P$ for some distribution over policies $P$. The
constraints (\ref{eq:var-bound}) are equivalent to
\eqref{eq:rucb-opt} by substitution $Z=W_Q$.


The above convex program can be solved by performing a binary search over $s$ and testing feasibility of the constraints.
For a fixed value of $s$, the feasibility problem defined by  \eqref{eq:low-regret}--\eqref{eq:var-bound} is denoted by $\mA$.

We now give a sketch of how we construct a separation oracle for the
feasible region of $\mA$. The details of the algorithm are a bit
complicated due to the fact that we need to ensure that the feasible
region, when non-empty, has a non-negligible volume (recall the
requirements of Lemma~\ref{lem:ellipsoid-description}). This necessitates
having a small error in satisfying the constraints of the program. We leave
the details to Appendix~\ref{sec:oracle-alg-details}%
.
Modulo these details, the construction of the separation oracle essentially implies that we can solve $\mA$.

Before giving the construction of the separation oracle, we first show that $\AMO$ allows us to do linear optimization over $\CH$ efficiently:
\begin{lem} \label{lem:linopt}
Given a vector $w \in \mathbb{R}^{(t-1)K}$, we can compute $\arg\max_{Z \in \CH} w
\cdot Z$ using one invocation of $\AMO$.
\end{lem}
\begin{proof}
The sequence for $\AMO$ consists of $x_{t'}\in \unlabeled_{t-1}$ and
$\vec{r}_{t'}(a)=w(x_{t'},a)$. The lemma now follows since
$w\cdot \pi = \sum_{x\in\unlabeled_{t-1}} w(x,\pi(x))$.
\end{proof}

We need another simple technical lemma which explains how to get a separating hyperplane for violations of convex constraints:
\begin{lem} \label{lem:sep-convex}
For $x \in \mathbb{R}^n$, let $f(x)$ be a convex function of $x$, and consider the convex set $K$ defined by $K = \{x:\ f(x) \leq 0\}$. Suppose we have a point $y$ such that $f(y) > 0$. Let $\nabla f(y)$ be a subgradient of $f$ at $y$. Then the hyperplane $f(y) + \nabla f(y) \cdot(x - y) = 0$ separates $y$ from $K$.
\end{lem}
\begin{proof}
Let $g(x) = f(y) + \nabla f(y) \cdot(x - y)$. By the convexity of $f$, we have $f(x) \geq g(x)$ for all $x$. Thus, for any $x \in K$, we have $g(x) \leq f(x) \leq 0$. Since $g(y) = f(y) > 0$, we conclude that $g(x) = 0$ separates $y$ from $K$.
\end{proof}

Now given a candidate point $W$, a separation oracle can be constructed as
follows. We check whether $W$ satisfies the constraints of $\mA$. If any
constraint is violated, then we find a hyperplane separating $W$ from all
points satisfying the constraint.
\begin{enumerate}
\item  First, for constraint (\ref{eq:low-regret}), note that $\heta(W)$ is
    linear in $W$, and so we can compute $\max_{\pi} \heta(\pi)$ via $\AMO$
    as in Lemma~\ref{lem:linopt}. We can then compute $\heta(W)$ and check
    if the constraint is satisfied. If not, then the constraint, being
    linear, automatically yields a separating hyperplane.

\item Next, we consider constraint~(\ref{eq:ch}). To check if $W \in \CH$,
    we use the perceptron algorithm. We shift the origin to $W$, and run
    the perceptron algorithm with all points $\pi \in \Pi$ being positive
    examples. The perceptron algorithm aims to find a hyperplane putting
    all policies $\pi \in \Pi$ on one side. In each iteration of the
    perceptron algorithm, we have a candidate hyperplane (specified by its
    normal vector), and then if there is a policy $\pi$ that is on the
    wrong side of the hyperplane, we can find it by running a linear
    optimization over $\CH$ in the negative normal vector direction as in
    Lemma~\ref{lem:linopt}.

    If $W \notin \CH$, then in a bounded number of iterations (depending on
    the distance of $W$ from $\CH$, and the maximum magnitude $\|\pi\|_2$) we
    obtain a separating hyperplane. In passing we also note that if $W \in
    \CH$, the same technique allows us to explicitly compute an
    approximate convex combination of policies in $\Pi$ that yields $W$.
    This is done by running the perceptron algorithm as before and stopping
    after the bound on the number of iterations has been reached. Then we
    collect all the policies we have found in the run of the perceptron
    algorithm, and we are guaranteed that $W$ is close in distance to their
    convex hull. We can then find the closest point in the convex hull of
    these policies by solving a simple quadratic program.

\item Finally, we consider constraint~(\ref{eq:var-bound}). We
    rewrite $\heta(W)$ as $\heta(W) = w \cdot W$, where 
    $w(x_{t'}, a) =r_{t'}\I(a=a_{t'})/W'_{t'}(a_{t'})$.
    Thus, $\Delta_{t-1}(Z) = v - w \cdot Z$, where $v =
    \max_{\pi'} \heta(\pi') = \max_{\pi'} w\cdot \pi'$, which can be
    computed by using $\AMO$ once.

    Next, using the candidate point $W$, compute the vector $u$ defined as
    $u(x, a) = \frac{n_x/t}{W'(x, a)}$, where $n_x$ is the number of times
    $x$ appears in $\history_{t-1}$, so that $\E_{x \sim \history_{t-1}}\left[\sum_a \frac{Z(x,
    a)}{W'(x,a)}\right] = u \cdot Z$. Now, the problem reduces to finding a
    policy $Z \in \CH$ which violates the constraint
    $$u \cdot Z \leq \max\{4K, \beta_t(w \cdot Z - v)^2\}.$$

    Define $f(Z) = \max\{4K, \beta_t(w \cdot Z - v)^2\} - u \cdot Z$. Note
    that $f$ is a convex function of $Z$. Finding a point $Z$ that violates
    the above constraint is equivalent to solving the following (convex)
    program:
    \begin{align}
    f(Z)\ &\leq\ 0 \label{eq:viol-sep-sketch}\\
    Z\ &\in\ \CH \label{eq:ch-sep-sketch}
    \end{align}
    To do this, we again apply the ellipsoid method. For this, we need
    a separation oracle for the program. A separation oracle for the
    constraints (\ref{eq:ch-sep-sketch}) can be constructed as in Step
    2 above. For the constraints (\ref{eq:viol-sep-sketch}), if the
    candidate solution $Z$ has $f(Z) > 0$, then we can construct a
    separating hyperplane as in Lemma~\ref{lem:sep-convex}.


    Suppose that after solving the program, we get a point $Z \in \CH$ such
    that $f(Z) \leq 0$, i.e. $W$ violates the
    constraint~(\ref{eq:var-bound}) for $Z$. Then since
    constraint~(\ref{eq:var-bound}) is convex in $W$, we can construct a
    separating hyperplane as in Lemma~\ref{lem:sep-convex}. This completes
    the description of the separation oracle.
\end{enumerate}

Working out the details carefully yields the following theorem,
proved in Appendix~\ref{sec:oracle-alg-details}%
:
\begin{thm} \label{thm:ellipsoid}
There is an iterative algorithm with
$O(t^5K^4\log^2(\frac{tK}{\delta}))$ iterations, each involving one call to
$\AMO$ and $O(t^2K^2)$ processing time, that either declares correctly
that $\mA$ is infeasible or outputs a
distribution $P$ over policies in $\Pi$ such that $W_P$ satisfies
\begin{gather*}
\forall Z \in \CH: \qquad \qquad \qquad \qquad \\
\E_{x \sim \history_{t-1}}\left[\sum_a \frac{Z(x, a)}{W_P'(x,a)}\right] \leq \max\{4K, \beta_t\Delta_{t-1}(Z)^2\} + 5\epsilon\\
\Delta_{t-1}(W)\ \leq\ s + 2\gamma,
\end{gather*}
where $\epsilon = \frac{8\delta}{\mu_t^2}$ and $\gamma = \frac{\delta
}{\mu_t}$.
\end{thm}

\section{DELAYED FEEDBACK}

\label{sec:delay}

In a delayed feedback setting, we observe rewards with a $\tau$ step
delay according to:
\begin{enumerate}
\item The world presents features $x_{t}$.
\item The learning algorithm chooses an action $a_{t}\in\{1,...,K\}$.
\item The world presents a reward $r_{t-\tau}$ for the action $a_{t-\tau}$
given the features $x_{t-\tau}$.
\end{enumerate}

We deal with delay by 
suitably modifying Algorithm~\ref{alg:PE} to incorporate the delay $\tau$, giving Algorithm~\ref{alg:PEdelay}.


\begin{algorithm}
\caption{$\DPE$($\Pi$,$\delta$,$K$,$D_X$,$\tau$)}\label{alg:PEdelay}

Let $\Pi_{0}=\Pi$
and history $h_{0}=\emptyset$
\\
Define:
$\delta_t \doteq \delta\,/\,4Nt^2$ and 
$\displaystyle b_{t} \doteq 2\sqrt{\frac{2K\ln(1/\delta_t)}{t}}$
\\
Define:
$\displaystyle \mu_t \doteq \min\Braces{\frac{1}{2K}\,,\,\sqrt{\frac{\ln(1/\delta_t)}{2Kt}}}$
\\
For each timestep $t= 1 \ldots T$, observe $x_t$ and do:
\begin{enumerate}
\item Let $t' = \max(t-\tau,1).$
\item Choose  distribution $P_t$ over $\Pi_{t-1}$ s.t.\ $\forall\ \pi\in\Pi_{t-1}$:
\label{a:stepPCD}
          \[
          \E_{x\sim D_X}\left[\frac{1}{(1-K\mu_{t'})W_{P_{t}}(x,\pi(x)) +
              \mu_{t'}}\right] \le 2K
          \]
\item $\forall\ a\in A,$ Let $W'_t(a)=(1-K\mu_{t'})W_{P_t}(x_{t},a) + \mu_{t'}$
\item Choose $a_t\sim W'_t$
\item Observe reward $r_t$.
\item Let $\displaystyle
    \Pi_{t}=\Bigl\{
    \pi\in\Pi_{t-1}:$
\\\hphantom{Let $\Pi_{t}=\Bigl\{$}
   $\displaystyle
    \eta_{h}(\pi)\ge\BigParens{\max_{\pi'\in
    \Pi_{t-1}}\!\!\!\eta_{h}(\pi')} - 2b_{t'}\Bigr\}$
    \label{a:stepElimD}
\item Let $h_{t}=h_{t-1} \cup (x_t,a_t,r_t,W'_t(a_t))$
\end{enumerate}

\end{algorithm}

Now we can prove the following theorem, which shows the delay has an additive effect on regret.

\begin{thm}\label{thm:PEdel}
For all distributions $D$ over $(x,\vec{r})$ with $K$ actions, for all
sets of $N$ policies $\Pi$, and all delay intervals $\tau$,
with probability at least $1-\delta$,
the regret of $\DPE$ (Algorithm~\ref{alg:PEdelay})
is at most
 \[
16\sqrt{2K\ln\frac{4T^2N}{\delta}}\left(\tau+\sqrt{T}\right).
\]
\end{thm}

\begin{proof}
Essentially as Theorem~\ref{thm:regretbd}.
The variance bound is unchanged
because it depends only on the context distribution. Thus, it suffices
to replace $\sum_{t-1}^{T}\frac{1}{\sqrt{t}}$
with $\tau+\sum_{t=\tau+1}^{T+\tau}\frac{1}{\sqrt{t-\tau}}=\tau+\sum_{t=1}^{T}\frac{1}{\sqrt{t}}$
in Eq.~\eqref{eq:sumrtt}.
\end{proof}

\subsubsection*{Acknowledgements}
We thank Alina Beygelzimer, who helped in several formative discussions.

\subsubsection*{References}
{\def\section*#1{}\small \bibliography{iid_contextual}
  \bibliographystyle{plainnat}} \appendix
\vspace{-0.00cm}
\section{Concentration Inequality} \label{sec:concentration}

The following is an immediate corollary of Theorem 1 of
\citep{EXP4P}. It can be viewed as a version of Freedman's
Inequality~\citep{Martingale}.
Let $y_{1},\ldots,y_{T}$ be a sequence of real-valued random variables.
Let $\Et$ denote the conditional expectation
$\E[{}\cdot{}|\,y_1,\ldots,y_{t-1}]$
and $\Vart$ conditional variance.
\begin{thm}[Freedman-style Inequality]
\label{thm:freedmanvar}
Let $V,R\in\R$ such that $\sum_{t=1}^T \Vart[y_t]\le V$, and for all $t$,
$y_t-\Et[y_t]\le R$. Then for any $\delta>0$ such that
$R\le\sqrt{V/\ln(2/\delta)}$,
with probability at least $1-\delta$,
\[
  \Abs{\sum_{t=1}^T y_t -\sum_{t=1}^T \Et[y_t]} \le
  2\sqrt{V\ln(2/\delta)}
\enspace.
\]
\end{thm}

\section{Minimax Theorem} \label{sec:sion}

The following is a continuous version of Sion's Minimax Theorem
\citep[Theorem~3.4]{Sion58}.

\begin{thm}
\label{thm:sion}
Let $\Wset$ and $\Zset$ be compact and convex sets, and
$f:\Wset\times\Zset\to\R$ a function which for all $Z\in\Zset$ is convex and continuous in $W$ and for all $W\in\Wset$ is concave
and continuous in $Z$. Then
\[
    \min_{W\in\Wset} \max_{Z\in\Zset} f(W,Z)
=
    \max_{Z\in\Zset} \min_{W\in\Wset} f(W,Z)
\enspace.
\]
\end{thm}

\section{Empirical Variance Bounds}
\label{sec:finite-appendix}

In this section we prove Theorem~\ref{thm:unlabeled}.
We first show uniform convergence for a certain class of policy
distributions (Lemma~\ref{lem:cover}), and argue that each
distribution $P$ is close to some distribution $\wt{P}$ from this
class, in the sense that $V_{P,\pi,t}$ is close to $V_{\wt{P},\pi,t}$
and $\wh{V}_{P,\pi,t}$ is close to $\wh{V}_{\wt{P},\pi,t}$
(Lemma~\ref{lem:sparsify-exp}).  Together, they imply the main
uniform convergence result in Theorem~\ref{thm:unlabeled}.

For each positive integer $m$, let $\Sparse{m}$ be the set of distributions
$\wt{P}$ over $\Pi$ that can be written as
\[ \wt{P}(\pi) = \frac1m \sum_{i=1}^m \I(\pi=\pi_i) \]
(i.e., the average of $m$ delta functions)
for some $\pi_1,\dotsc,\pi_m \in \Pi$.
In our analysis, we approximate an arbitrary distribution $P$ over $\Pi$ by
a distribution $\wt{P} \in \Sparse{m}$ chosen randomly by independently
drawing $\pi_1,\dotsc,\pi_m \sim P$; we denote this process by $\wt{P} \sim
P^m$.

\begin{lem} \label{lem:cover}
Fix positive integers $(m_1,m_2,\dotsc)$.
With probability at least $1-\delta$ over the random samples
$(x_1,x_2,\dotsc)$ from $D_X$,
\begin{multline*}
V_{\wt{P},\pi,t}
\leq
(1+\lambda) \cdot \wh{V}_{\wt{P},\pi,t}
\\
+ \left(5+\frac1{2\lambda} \right)
\cdot \frac{(m_t+1) \log N + \log\frac{2t^2}{\delta}}{\mu_t \cdot (t-1)}
\end{multline*}
for all $\lambda > 0$, all $t \geq 1$, all $\pi \in \Pi$, and all
distributions $\wt{P} \in \Sparse{m_t}$.
\end{lem}

\begin{proof}
Let
\[
Z_{\wt{P},\pi,t}(x) \doteq \frac1{(1-K\mu_t) W_{\wt{P}}(x,\pi(x)) +
\mu_t}
\]
so
$V_{\wt{P},\pi,t} = \E_{x\sim D_X}[Z_{\wt{P},\pi,t}(x)]$
and
$\wh{V}_{\wt{P},\pi,t} = (t-1)^{-1} \sum_{i=1}^{t-1} Z_{\wt{P},\pi,t}(x_i)$.
Also let
\begin{align*}
\veps_t
& \doteq \frac{\log(|\Sparse{m_t}|N2t^2/\delta)}{\mu_t \cdot (t-1)}
\\
& =
\frac{( (m_t+1) \log N + \log\frac{2t^2}{\delta})}{\mu_t \cdot (t-1)}
.
\end{align*}
We apply Bernstein's inequality and union bounds over $\wt{P} \in
\Sparse{m_t}$, $\pi \in \Pi$, and $t \geq 1$ so that with
probability at least $1-\delta$,
\[
V_{\wt{P},\pi,t}
\leq
\wh{V}_{\wt{P},\pi,t}
+ \sqrt{2V_{\wt{P},\pi,t} \veps_t}
+ (2/3) \veps_t
\]
all $t \geq 1$, all $\pi \in \Pi$, and all distributions $P \in
\Sparse{m_t}$.
The conclusion follows by solving the quadratic inequality for
$V_{\wt{P},\pi,t}$ to get
\[
V_{\wt{P},\pi,t}
\leq
\wh{V}_{\wt{P},\pi,t}
+ \sqrt{2\wh{V}_{\wt{P},\pi,t} \veps_t}
+ 5 \veps_t
\]
and then applying the AM/GM inequality.
\end{proof}

\begin{lem} \label{lem:sparsify-exp}
Fix any $\gamma \in [0,1]$, and any $x \in X$.
For any distribution $P$ over $\Pi$ and any $\pi \in \Pi$, if
\[ m
\doteq \left\lceil
\frac{6}{\gamma^2\mu_t}
\right\rceil
,
\]
then
\begin{multline*}
\E_{\wt{P} \sim P^m}
\Biggl|\frac1{(1-K\mu_t) W_{\wt{P}}(x,\pi(x)) + \mu_t}
\\
 \qquad{} - \frac1{(1-K\mu_t) W_P(x,\pi(x)) + \mu_t}
\Biggr| \\
\leq \frac{\gamma}{(1-K\mu_t) W_P(x,\pi(x)) + \mu_t}
.
\end{multline*}
This implies that for all distributions $P$ over $\Pi$ and any $\pi \in \Pi$,
there exists $\wt{P} \in \Sparse{m}$ such that for any $\lambda > 0$,
\begin{multline*}
 \left( V_{P,\pi,t} - V_{\wt{P},\pi,t} \right)
+(1+\lambda) \left( \wh{V}_{\wt{P},\pi,t} - \wh{V}_{P,\pi,t} \right) \\
 \leq \gamma (V_{P,\pi,t}+(1+\lambda) \wh{V}_{P,\pi,t}) .
\end{multline*}
\end{lem}

\begin{proof}
We randomly draw $\wt{P} \sim P^m$, with $\wt{P}(\pi') \doteq m^{-1}
\sum_{i=1}^m \I(\pi' = \pi_i)$, and then define
\begin{align*}
z & \doteq \sum_{\pi'\in\Pi} P(\pi') \cdot \I(\pi'(x) = \pi(x))
\quad \text{and}
\\
\hat{z} & \doteq \sum_{\pi'\in\Pi} \wt{P}(\pi') \cdot \I(\pi'(x) =
\pi(x))
.
\end{align*}
We have
$z = \E_{\pi'\sim P}[\I(\pi'(x) = \pi(x)]$ and
$\hat{z} = m^{-1} \sum_{i=1}^m \I(\pi_i(x) = \pi(x))$.
In other words, $\hat{z}$ is the average of $m$ independent Bernoulli
random variables, each with mean $z$.
Thus, $\E_{\wt{P}\sim P^m}[(\hat{z}-z)^2] = z(1-z)/m$ and
$\Pr_{\wt{P}\sim P^m}[\hat{z} \leq z/2] \leq \exp(-mz/8)$ by a Chernoff
bound.
We have
\begin{align*}
\lefteqn{
\E_{\wt{P} \sim P^m}
\left|\frac1{(1-K\mu_t)\hat{z} + \mu_t}
-\frac1{(1-K\mu_t)z + \mu_t}\right|
}
\\
&\leq
\E_{\wt{P} \sim P^m}
\frac{(1-K\mu_t)|\hat{z}-z|}{[(1-K\mu_t)\hat{z} +
\mu_t][(1-K\mu_t)z + \mu_t]} \\
&\leq
\E_{\wt{P} \sim P^m}
\frac{(1-K\mu_t) |\hat{z}-z| \I(\hat{z} \geq 0.5 z)}
{0.5 [(1-K\mu_t)z + \mu_t]^2} \\
&\quad{}
+ \E_{\wt{P} \sim P^m}
\frac{(1-K\mu_t)|\hat{z}-z| \I(\hat{z} \leq 0.5
z)}{\mu_t[(1-K\mu_t)z + \mu_t]} \\
&\leq
\frac{(1-K\mu_t) \sqrt{\E_{\wt{P} \sim P^m} |\hat{z}-z|^2}}{0.5
[(1-K\mu_t)z + \mu_t]^2} \\
&\quad{}
+ \frac{(1-K\mu_t)z \Pr_{\wt{P}\sim P^m}(\hat{z} \leq 0.5
z)}{\mu_t[(1-K\mu_t)z + \mu_t]} \\
&\leq
\frac{(1-K\mu_t) \sqrt{z/m}}{0.5 [2\sqrt{(1-K\mu_t)z \mu_t}]
[(1-K\mu_t)z + \mu_t]} \\
&\quad{}
+ \frac{(1-K\mu_t)z \exp(-m z/8)}{\mu_t[(1-K\mu_t)z + \mu_t]} \\
&\leq
\frac{\gamma \sqrt{1-K\mu_t} \sqrt{z/m}}{\sqrt{z (6/m)} [(1-K\mu_t)z +
\mu_t]} \\
&\quad{}
+ \frac{(1-K\mu_t) \gamma^2 m z \exp(-m z/8)}{6 [(1-K\mu_t)z + \mu_t]} ,
\end{align*}
where the third inequality follows from Jensen's inequality, and the fourth
inequality uses the AM/GM inequality in the denominator of the first term
and the previous observations in the numerators.
The final expression simplifies to the first desired displayed inequality
by observing that $m z \exp(-m z/8) \leq 3$ for all $mz \geq 0$ (the
maximum is achieved at $mz=8$).
The second displayed inequality follows from the following facts:
\begin{align*}
&\E_{\wt{P} \sim P^m}  |V_{P,\pi,t} - V_{\wt{P},\pi,t}| \leq \gamma V_{P,\pi,t} ,\\
&\E_{\wt{P} \sim P^m} (1+\lambda) | \wh{V}_{P,\pi,t} - \wh{V}_{\wt{P},\pi,t}
| \leq \gamma (1+\lambda) \wh{V}_{P,\pi,t} .
\end{align*}
Both inequalities follow from the first displayed bound of the lemma,
by taking expectation with respect to the true (and empirical) distributions over $x$.
The desired bound follows by adding the above two inequalities, which implies that the bound holds in expectation,
and hence the existence of $\wt{P}$ for which the bound holds.
\end{proof}

Now, we can prove Theorem~\ref{thm:unlabeled}.

\begin{proof}[Proof of Theorem~\ref{thm:unlabeled}]
Let
\[
m_t \doteq
\left\lceil \frac{6}{\lambda^2} \cdot \frac{1}{\mu_t} \right\rceil
\]
(for some $\lambda \in (0,1/5)$ to be determined)
and condition on the $\geq 1-\delta$ probability event from
Lemma~\ref{lem:cover} that
\begin{multline*}
V_{\wt{P},\pi,t} - (1+\lambda) \wh{V}_{\wt{P},\pi,t}
\\
\leq K \cdot \left(5 + \frac1{2\lambda}\right) \cdot \frac{(m_t+1)\log(N)
+ \log(2t^2/\delta)}{K\mu_t \cdot (t-1)}
\\
\leq K \cdot 5\left(1 + \frac1{\lambda}\right) \cdot \frac{(m_t+1)\log(N)
+ \log(2t^2/\delta)}{K\mu_t \cdot t}
\end{multline*}
for all $t \geq 2$, all $\wt{P} \in \Sparse{m_t}$, and all $\pi \in
\Pi$.
Using the definitions of $m_t$ and $\mu_t$, the second term is at most
$(40 / \lambda^2) \cdot (1+1/\lambda) \cdot K$ for all $t \geq
16K\log(8KN/\delta)$:
the key here is that for $t \geq 16 K \log(8KN/\delta)$, we have
$\mu_t = \sqrt{\log(Nt/\delta)/(Kt)} \leq 1/(2K)$ and therefore
\[
\frac{m_t \log(N)}{K\mu_t t}
\leq \frac{6}{\lambda^2}
\quad \text{and} \quad
\frac{\log(N) + \log(2t^2/\delta)}{K\mu_t t} \leq 2 . \]

Now fix $t \geq 16K\log(8KN/\delta)$, $\pi \in \Pi$, and a
distribution $P$ over $\Pi$.
Let $\wt{P} \in \Sparse{m_t}$ be the distribution guaranteed by
Lemma~\ref{lem:sparsify-exp} with $\gamma = \lambda$ satisfying
\[ V_{P,\pi,t} \leq \frac{V_{\wt{P},\pi,t} - (1+\lambda)
\wh{V}_{\wt{P},\pi,t} + (1+\lambda)^2 \wh{V}_{P,\pi,t}}{1-\lambda} . \]
Substituting the previous bound for $V_{\wt{P},\pi,t} - (1+\lambda)
\wh{V}_{\wt{P},\pi,t}$ gives
\begin{multline*}
V_{P,\pi,t}
\leq \frac1{1-\lambda} \left(
\frac{40}{\lambda^2} (1 + 1/\lambda) K
+ (1+\lambda)^2 \wh{V}_{P,\pi,t}
\right)
.
\end{multline*}
This can be bounded as $(1 + \epsilon) \cdot \wh{V}_{P,\pi,t} +
(7500/\epsilon^3) \cdot K$ by setting $\lambda = \epsilon/5$.
\end{proof}

\section{Analysis of $\RUCB$} \label{sec:rucb-appendix}
\def\mixtures{\mathcal{D}(\Pi)}

%
%
%
%

\subsection{Preliminaries}

First, we define the following constants.
\begin{itemize}
    \item $\epsilon \in (0,1)$ is a fixed constant, and

    \item $\rho \doteq \frac{7500}{\epsilon^3}$ is the factor that appears
    in the bound from Theorem~\ref{thm:unlabeled}.

    \item $\theta \doteq (\rho+1) / (1 - (1+\epsilon)/2) =
    \frac{2}{1-\epsilon} \left(1 + \frac{7500}{\epsilon^3} \right) \geq 5$
    is a constant central to Lemma~\ref{lem:vbound}, which bounds the
    variance of the optimal policy's estimated rewards.

\end{itemize}

Recall the algorithm-specific quantities
\begin{align*}
C_t & \doteq 2\log\left(\frac{Nt}{\delta}\right) \\
\mu_t & \doteq \min\left\{\frac1{2K}, \ \sqrt{\frac{C_t}{2Kt}} \right\}
.
\end{align*}
It can be checked that $\mu_t$ is non-increasing.
We define the following time indices:
\begin{itemize}
\item $t_0$ is the first round $t$ in which $\mu_t = \sqrt{C_t/(2Kt)}$.
Note that $8K \leq t_0 \leq 8K\log(NK/\delta)$.

\item $t_1 := \lceil 16K\log(8KN/\delta) \rceil$ is the round given by
Theorem~\ref{thm:unlabeled} such that, with probability at least
$1-\delta$,
\begin{multline} \label{eq:variance-bound}
\E_{x_t \sim D_X}\left[ \frac1{W_t'(\pi(x_t))} \right]
\\
 \leq (1+\epsilon) \E_{x \sim h_{t-1}}\left[ \frac1{W_{P_t,\mu_t}(x,\pi(x))}
 \right]
+ \rho K
\end{multline}
for all $\pi \in \Pi$ and all $t \geq t_1$,
where $W_{P,\mu}(x,\cdot)$ is the distribution over $A$ given by
\[
W_{P,\mu}(x,a) \doteq (1-K\mu) W_P(x,a) + \mu
,
\]
and the notation $\E_{x \sim h_{t-1}}$ denotes expectation with respect to
the empirical (uniform) distribution over $x_1,\dotsc,x_{t-1}$.
\end{itemize}

The following lemma shows the effect of allowing slack in the optimization
constraints.
\begin{lem} \label{lem:approx-opt}
If $P$ satisfies the constraints of the optimization
problem~\eqref{eq:rucb-opt} with slack $K$ for each distribution $Q$ over
$\Pi$, i.e.,
\begin{multline*}
\E_{\pi \sim Q}
\E_{x \sim h_{t-1}} \left[ \frac{1}{(1-K\mu_t) W_{P}(x, \pi(x)) + \mu_t} \right]
\\
\leq \max\left\{ 4K, \frac{(t-1) \Delta_{t-1}(W_Q)^2}{180 C_{t-1}} \right\} +
K
\end{multline*}
for all $Q$, then $P$ satisfies
\begin{multline*}
\E_{\pi \sim Q}
\E_{x \sim h_{t-1}} \left[ \frac{1}{(1-K\mu_t) W_{P}(x, \pi(x)) + \mu_t} \right]
\\
\leq \max\left\{ 5K, \frac{(t-1) \Delta_{t-1}(W_Q)^2}{144 C_{t-1}}
\right\}
\end{multline*}
for all $Q$.
\end{lem}
\begin{proof}
Let $b \doteq \max\left\{ 4K, \frac{(t-1) \Delta_{t-1}(\pi)^2}{180 C_{t-1}}
\right\}$. Note that $\frac{b}{4} \geq K$. Hence $b + K
\leq \frac{5b}{4}$ which gives the stated bound.
\end{proof}
Note that the allowance of slack $K$ is somewhat arbitrary; any $O(K)$
slack is tolerable provided that other constants are adjusted
appropriately.

\subsection{Deviation Bound for $\eta_t(\pi)$}

For any policy $\pi \in \Pi$, define, for $1 \leq t \leq t_0$,
\[
\vbound{t}(\pi) \doteq K
,
\]
and for $t > t_0$,
\[
\vbound{t}(\pi) \doteq K + \E_{x_t \sim D_X}\left[ \frac1{W_t'(\pi(x_t))} \right].
\]
The $\vbound{t}(\pi)$ bounds the variances of the terms in $\eta_t(\pi)$.
\begin{lem} \label{lem:t1}
Assume the bound in~\eqref{eq:variance-bound} holds for all $\pi \in \Pi$
and $t \geq t_1$.
For all $\pi \in \Pi$:
\begin{enumerate}
\item If $t \leq t_1$, then
\[ K \leq \vbound{t}(\pi) \leq 4K
. \]

\item If $t > t_1$, then
\begin{align*}
\lefteqn{
\vbound{t}(\pi)
}
\\
& \leq
(1+\epsilon)
  \E_{x \sim h_{t-1}} \left[
    \frac1{(1-K\mu_t) W_{P_t}(x,\pi(x)) + \mu_t}
  \right]
  \\
& \quad{}
+ (\rho+1) K
. \end{align*}
\end{enumerate}
\end{lem}
\begin{proof}
For the first claim, note that if $t < t_0$, then $\vbound{t}(\pi) = K$,
and if $t_0 \leq t < t_1$, then
\[
\mu_t
= \sqrt{\frac{\log(Nt/\delta)}{Kt}}
\geq \sqrt{\frac{\log(Nt_0/\delta)}{16K^2\log(8KN/\delta)}}
\geq \frac1{4K}
;
\]
so $W_t'(a) \geq \mu_t \geq 1/(4K)$.

For the second claim, pick any $t > t_1$, and note that by definition of $t_1$, for any $\pi \in \Pi$ we have
\begin{align*}
\lefteqn{
\E_{x_t \sim D_X}\left[ \frac1{W_t'(\pi(x_t))} \right]
} \\
& \leq (1+\epsilon) \E_{x \sim h_{t-1}}\left[ \frac1{(1-K\mu_t) W_{P_t}(x, \pi(x))
+ \mu_t} \right]
+ \rho K
.
\end{align*}
The stated bound on $\vbound{t}(\pi)$ now follows from its definition.
\end{proof}

Let
\begin{equation*}
\vmax{t}(\pi) \doteq \max\{\vbound{\tau}(\pi),\ \tau = 1, 2, \ldots, t\} 
\end{equation*}
The following lemma gives a deviation bound for $\eta_t(\pi)$ in terms of these
quantities.
\begin{lem} \label{lem:deviation-bound}
Pick any $\delta \in (0,1)$.
With probability at least $1-\delta$, for all pairs $\pi,\pi'\in\Pi$ and $t
\geq t_0$, we have
\begin{multline} \label{eq:deviation-bound}
\Bigl|(\eta_t(\pi)-\eta_t(\pi')) - (\eta_D(\pi)-\eta_D(\pi'))\Bigr|
\\
\leq 2\sqrt{\frac{ (\vmax{t}(\pi) + \vmax{t}(\pi')) \cdot C_t}{t}}.
\end{multline}
\end{lem}
\begin{proof}
Fix any $t \geq t_0$ and $\pi,\pi' \in \Pi$.
Let $\delta_t := \exp(-C_t)$.
Pick any $\tau \leq t$.
Let
\[
Z_\tau(\pi)
\doteq \frac{r_\tau(a_\tau)\I(\pi(x_\tau) = a_\tau)}{W_\tau'(a_\tau)}
\]
so $\eta_t(\pi) = t^{-1} \sum_{\tau=1}^t Z_\tau(\pi)$.
It is easy to see that
\[
\E_{\substack{(x_\tau,\vec{r}_\tau) \sim D, \\ a_\tau \sim W_\tau'}}
\left[ Z_\tau(\pi) - Z_\tau(\pi') \right]
= \eta_D(\pi) - \eta_D(\pi')
\]
and
\begin{align*}
\lefteqn{
\sum_{\tau=1}^t \E_{\substack{(x_\tau,\vec{r}(\tau)) \sim D, \\ a_\tau \sim
W_\tau'}}
\left[ (Z_\tau(\pi) - Z_\tau(\pi'))^2 \right]
}
\\
& \leq \sum_{\tau=1}^t \E_{x_\tau \sim D_X}\left[ \frac1{W_\tau'(\pi(x_\tau))} +
\frac1{W_\tau'(\pi'(x_\tau))}
\right]
\\
&\leq t\cdot(\vmax{t}(\pi) + \vmax{t}(\pi'))
.
\end{align*}
Moreover, with probability $1$,
\[
|Z_\tau(\pi) - Z_\tau(\pi')| \leq \frac{1}{\mu_\tau}
.
\]
Now, note that since $t \geq t_0$, $\mu_t = \sqrt{\frac{C_t}{2Kt}}$, so
that $t = \frac{C_t}{2K\mu_t^2}$.
Further, both $\vmax{t}(\pi)$ and $\vmax{t}(\pi')$ are at least $K$.
Using these bounds we get
\begin{align*}
&\sqrt{\frac{1}{\log(1/\delta_t)} \cdot t\cdot(\vmax{t}(\pi) + \vmax{t}(\pi'))}
\\
& \geq \sqrt{\frac{1}{C_t} \cdot \frac{C_t}{2K\mu_t^2} \cdot
2K}
= \frac1{\mu_t} \geq \frac1{\mu_\tau},
\end{align*}
for all $\tau \leq t$, since the $\mu_\tau$'s are non-increasing.
Therefore, by Freedman's inequality (Theorem~\ref{thm:freedmanvar}), we have
\begin{multline*}
\Pr\Biggl[
\Bigl|(\eta_t(\pi)-\eta_t(\pi')) - (\eta_D(\pi)-\eta_D(\pi'))\Bigr|
\\
> 2\sqrt{\frac{ (\vmax{t}(\pi) + \vmax{t}(\pi')) \cdot  \log(1/\delta_t)}{t}}
\Biggr]
\leq 2\delta_t
.
\end{multline*}
The conclusion follows by taking a union bound over $t_0 < t \leq T$ and
all pairs $\pi,\pi' \in \Pi$.
\end{proof}

\subsection{Variance Analysis}

We define the following condition, which will be assumed by most of the
subsequent lemmas in this section.
\begin{cond} \label{cond:rucb}
The deviation bound~\eqref{eq:variance-bound} holds for all $\pi \in \Pi$
and $t \geq t_1$, and
the deviation bound~\eqref{eq:deviation-bound} holds for all pairs
$\pi,\pi' \in \Pi$ and $t \geq t_0$.
\end{cond}

The next two lemmas relate the $\vbound{t}(\pi)$ to the $\Delta_t(\pi)$.
\begin{lem} \label{lem:constraint}
Assume Condition~\ref{cond:rucb}.
For any $t \geq t_1$ and $\pi \in \Pi$, if $\vbound{t}(\pi) > \theta  K$, then
\[
\Delta_{t-1}(\pi)
\geq \sqrt{\frac{72\vbound{t}(\pi)C_{t-1}}{t-1}}
.
\]
\end{lem}
\begin{proof}
By Lemma~\ref{lem:t1}, the fact $\vbound{t}(\pi) > \theta  K$ implies that
\begin{multline*}
\E_{x \sim h_{t-1}}\left[ \frac1{(1-K\mu_t)W_{P_t}(x,\pi(x)) + \mu_t} \right]
\\
> \frac1{1+\epsilon} \left(1 - \frac{\rho+1}{\theta} \right) \vbound{t}(\pi)
\geq \frac{1}{2}\vbound{t}(\pi)
.
\end{multline*}
Since $\vbound{t}(\pi) > \theta K \geq 5K$, Lemma~\ref{lem:approx-opt}
implies that in order for $P_t$ to satisfy the optimization constraint
in~\eqref{eq:rucb-opt} corresponding to $\pi$ (with slack $\leq K$), it
must be the case that
\begin{multline*}
\Delta_{t-1}(\pi)
\\
\geq \sqrt{\frac{144 C_{t-1}}{t-1} \cdot
\E_{x \sim h_{t-1}}\left[ \frac1{(1-K\mu_t)W_{P_t}(x,\pi(x)) + \mu_t} \right]
}
.
\end{multline*}
Combining with the above, we obtain
\[
\Delta_{t-1}(\pi)
\geq \sqrt{\frac{72\vbound{t}(\pi)C_{t-1}}{t-1}}
.
\]
\end{proof}

\begin{lem} \label{lem:vbound}
Assume Condition~\ref{cond:rucb}.
For all $t \geq 1$, $\vmax{t}(\pimax) \leq \theta K$ and $\vmax{t}(\pi_t)
\leq \theta K$.
\end{lem}
\begin{proof}
By induction on $t$. The claim for all $t \leq t_1$ follows from
Lemma~\ref{lem:t1}. So take $t > t_1$, and assume as the (strong) inductive
hypothesis that $\vmax{\tau}(\pimax) \leq \theta  K$ and $\vmax{\tau}(\pi_\tau)
\leq \theta  K$ for $\tau \in \{1,\dotsc,t-1\}$. Suppose for sake of
contradiction that $\vbound{t}(\pimax) > \theta  K$. By
Lemma~\ref{lem:constraint},
\[
\Delta_{t-1}(\pimax) \geq \sqrt{\frac{72\vbound{t}(\pimax)C_{t-1}}{t-1}} . \]
However, by the deviation bounds, we have
\begin{align*}
\lefteqn{
\Delta_{t-1}(\pimax) + \Delta_D(\pi_{t-1})
}
\\
& \leq 2\sqrt{\frac{(\vmax{t-1}(\pi_{t-1}) + \vmax{t-1}(\pimax)) C_{t-1}}{t-1}}
\\
& \leq 2\sqrt{\frac{2\vbound{t}(\pimax) C_{t-1}}{t-1}}
< \sqrt{\frac{72\vbound{t}(\pimax)C_{t-1}}{t-1}}
.
\end{align*}
The second inequality follows from our assumption and the induction hypothesis:
$$\vbound{t}(\pimax)
> \theta   K \geq \vmax{t-1}(\pi_{t-1}), \vmax{t-1}(\pimax).$$
Since $\Delta_D(\pi_{t-1}) \geq 0$, we have a contradiction, so it must be that
$\vbound{t}(\pimax) \leq \theta  K$. This proves that $\vmax{t}(\pimax) \leq
\theta  K$.

It remains to show that $\vmax{t}(\pi_t) \leq \theta  K$. So suppose for sake
of contradiction that the inequality fails, and let $t_1 < \tau \leq t$ be any
round for which $\vbound{\tau}(\pi_t) = \vmax{t}(\pi_t) > \theta  K$. By
Lemma~\ref{lem:constraint},
\begin{equation} \label{eq:gap-in-past}
\Delta_{\tau-1}(\pi_t)
\geq \sqrt{\frac{72 \vbound{\tau}(\pi_t) C_{\tau-1}}{\tau-1}}
.
\end{equation}
On the other hand,
\begin{align*}
\Delta_{\tau-1}(\pi_t)
& \leq \Delta_D(\pi_{\tau-1})
+ \Delta_{\tau-1}(\pi_t)
+ \Delta_t(\pimax)
\\
& =
\Bigl( \Delta_D(\pi_{\tau-1}) + \Delta_{\tau-1}(\pimax) \Bigr)
\\
& {}\quad
+ \Bigl( \eta_{\tau-1}(\pimax) - \eta_{\tau-1}(\pi_t) - \Delta_D(\pi_t)
\Bigr)
\\
& {}\quad
+ \Bigl( \Delta_D(\pi_t) + \Delta_t(\pimax) \Bigr)
.
\end{align*}
The parenthesized terms can be bounded using the deviation bounds, so we
have
\begin{align*}
& \Delta_{\tau-1}(\pi_t) \\
\leq &
2\sqrt{\frac{(\vmax{\tau-1}(\pi_{\tau-1}) + \vmax{\tau-1}(\pimax)) C_{\tau-1}}{\tau-1}}
\\
& \quad{} + 2\sqrt{\frac{(\vmax{\tau-1}(\pi_t) + \vmax{\tau-1}(\pimax)) C_{\tau-1}}{\tau-1}}
\\
& \quad{}
+ 2\sqrt{\frac{(\vmax{t}(\pi_t) + \vmax{t}(\pimax)) C_t}{t}}
\\
\leq &
2\sqrt{\frac{2\vbound{\tau}(\pi_t) C_{\tau-1}}{\tau-1}}
+ 2\sqrt{\frac{2\vbound{\tau}(\pi_t)C_{\tau-1}}{\tau-1}}
\\
& \quad{}
+ 2\sqrt{\frac{2\vbound{\tau}(\pi_t) C_t}{t}}
\\
< & \sqrt{\frac{72\vbound{\tau}(\pi_t) C_{\tau-1}}{\tau-1}}
\end{align*}
where the second inequality follows from the following facts:
\begin{enumerate*}
    \item By induction hypothesis, we have $\vmax{\tau-1}(\pi_{\tau-1}), \vmax{\tau-1}(\pimax),
    \vmax{t}(\pimax) \leq \theta   K$, and $\vbound{\tau}(\pi_t) > \theta   K$,
    \item $\vbound{\tau}(\pi_t) \geq \vmax{t}(\pi_t)$, and
    \item since $\tau$ is a round that achieves $\vmax{t}(\pi_t)$, we have $\vbound{\tau}(\pi_t)
    \geq \vbound{\tau-1}(\pi_t)$.
\end{enumerate*}
This contradicts the inequality in~\eqref{eq:gap-in-past}, so it must be that
$\vmax{t}(\pi_t) \leq \theta  K$.
\end{proof}

\begin{cor} \label{cor:vbound}
Under the assumptions of Lemma~\ref{lem:vbound},
\[
\Delta_D(\pi_t) + \Delta_t(\pimax)
\leq 2\sqrt{\frac{2\theta KC_t}{t}}
\]
for all $t \geq t_0$.
\end{cor}
\begin{proof}
Immediate from Lemma~\ref{lem:vbound} and the deviation bounds
from~\eqref{eq:deviation-bound}.
\end{proof}

The following lemma shows that if a policy $\pi$ has large $\Delta_\tau(\pi)$ in
some round $\tau$, then $\Delta_t(\pi)$ remains large in later rounds $t >
\tau$.
\begin{lem} \label{lem:monotone}
Assume Condition~\ref{cond:rucb}.
Pick any $\pi \in \Pi$ and $t \geq t_1$. If
$\vmax{t}(\pi) > \theta  K$, then
\[
\Delta_{t}(\pi) > 2\sqrt{\frac{2\vmax{t}(\pi) C_t}{t}}
. \]
\end{lem}
\begin{proof}
Let $\tau \leq t$ be any round in which $\vbound{\tau}(\pi) = \vmax{t}(\pi)
> \theta  K$. We have
\begin{align*}
\Delta_t(\pi)
& \geq
\Delta_t(\pi)
- \Delta_t(\pimax)
- \Delta_D(\pi_{\tau-1})
\\
& = \Delta_{\tau-1}(\pi)
+ \Bigl( \eta_t(\pimax) - \eta_t(\pi) - \Delta_D(\pi) \Bigr)
\\
& \quad{}
+ \Bigl( \eta_D(\pi_{\tau-1}) - \eta_D(\pi) - \Delta_{\tau-1}(\pi) \Bigr)
\\
& \geq
\sqrt{\frac{72\vbound{\tau}(\pi)C_{\tau-1}}{\tau-1}}
\\
& \quad{}
- 2\sqrt{\frac{(\vmax{t}(\pi) + \vmax{t}(\pimax)) C_t}{t}}
\\
& \quad{}
- 2\sqrt{\frac{(\vmax{\tau-1}(\pi) + \vmax{\tau-1}(\pi_{\tau-1})) C_{\tau-1}}{\tau-1}}
\\
& >
\sqrt{\frac{72\vmax{t}(\pi)C_{\tau-1}}{\tau-1}}
- 2\sqrt{\frac{2\vmax{t}(\pi) C_t}{t}}
\\
& \quad{}
- 2\sqrt{\frac{2\vmax{t}(\pi) C_{\tau-1}}{\tau-1}}
\\
& \geq 2\sqrt{\frac{2\vmax{t}(\pi) C_{\tau-1}}{\tau-1}}
\geq 2\sqrt{\frac{2\vmax{t}(\pi) C_t}{t}}
\end{align*}
where the second inequality follows from Lemma~\ref{lem:constraint} and the
deviation bounds, and the third inequality follows from Lemma~\ref{lem:vbound}
and the facts that $\vbound{\tau}(\pi) = \vmax{t}(\pi) > \theta  K \geq
\vmax{t}(\pimax), \vmax{\tau-1}(\pi_{\tau-1})$, and $\vmax{t}(\pi) \geq
\vmax{\tau-1}(\pi)$.
\end{proof}

\subsection{Regret Analysis}
\label{sec:rucb-appendix:regret}

We now bound the value of the optimization problem~\eqref{eq:rucb-opt},
which then leads to our regret bound.
The next lemma shows the existence of a feasible solution with a certain
structure based on the non-uniform constraints. Recall from Section~\ref{sec:oracle}, that solving the optimization problem $\mA$, i.e. constraints (\ref{eq:low-regret}, \ref{eq:ch}, \ref{eq:var-bound}), for the smallest feasible value of $s$ is equivalent to solving the RUCB optimization problem (\ref{eq:rucb-opt}). Recall that $\beta_t = \frac{t-1}{180C_{t-1}}$.

\begin{lem} \label{lem:value}
There is a point $W \in \mathbb{R}^{(t-1)K}$ such that
\begin{gather*}
\Delta_{t-1}(W)\ \leq\ 4\sqrt{\frac{K}{\beta_t}}\\
W\ \in\ \CH\\
\forall Z \in \CH: \!\!\!\!\E_{x \sim \history_{t-1}}\left[\sum_a \frac{Z(x, a)}{W'(x,a)}\right] \leq \max\{4K, \beta_t\Delta_{t-1}(Z)^2\}
\end{gather*}
In particular, the value of the optimization problem~(\ref{eq:rucb-opt}), $\OPT_t$, is bounded by $8\sqrt{\frac{K}{\beta_t}} \leq 110\sqrt{\frac{KC_{t-1}}{t-1}}$.
\end{lem}
\begin{proof}
Define the sets $\{\CH_i:\ i=1,2,\ldots\}$ such that
$$\CH_i := \{Z \in \CH:\ 2^{i+1}\kappa \leq \Delta_{t-1}(Z) \leq 2^{i+2}\kappa\},$$
where $\kappa = \sqrt{\frac{K}{\beta_t}}$. Note that since $\Delta_{t-1}(Z)$ is a linear function of $Z$, each $\CH_i$ is a closed, convex, compact set. Also, define $\CH_0 = \{Z \in \CH:\ \Delta_{t-1}(Z) \leq 4\kappa\}$. This is also a closed, convex, compact set. Note that $\CH = \bigcup_{i=0}^\infty \CH_i$.

Let $I = \{i:\ \CH_i \neq \emptyset\}$.For $i \in I \setminus \{0\}$, define $w_i = 4^{-i}$, and let $w_0 = 1 - \sum_{i \in I\setminus \{0\}} w_i$. Note that $w_0 \geq 2/3$.

By Lemma~\ref{lem:minimax}, for each $i \in I$, there is a point $W_i \in \CH_i$ such that for all $Z \in \CH_i$, we have
$$\E_{x \sim \history_{t-1}}\left[\sum_a \frac{Z(x, a)}{W_i'(x,a)}\right] \leq 2K.$$
Here we use the fact that $K\mu_t \leq 1/2$ to upper bound $\frac{K}{1-K\mu_t}$ by $2K$. Now consider the point $W = \sum_{i \in I} w_i W_i$. Since $\CH$ is convex, $W \in \CH$.

Now fix any $i \in I$. For any $(x, a)$, we have $W'(x, a) \geq w_iW'_i(x, a)$, so that for all $Z \in \CH_i$, we have
\begin{align*}
\E_{x \sim \history_{t-1}}\left[\sum_a \frac{Z(x, a)}{W'(x,a)}\right] &\leq \frac{1}{w_i}2K \\
&\leq 4^{i+1}K \\
&\leq \max\{4K, \beta_t \Delta_{t-1}(Z)^2\},
\end{align*}
so the constraint for $Z$ is satisfied.

Finally, since for all $i \in I$, we have $w_i \leq 4^{-i}$ and $\Delta_{t-1}(W_i) \leq 2^{i+2}\kappa$, we get
$$\Delta_{t-1}(W) = \sum_{i \in I} w_i \Delta_{t-1}(W_i) \leq \sum_{i=0}^\infty 4^{-i}\cdot 2^{i+2}\kappa \leq 8\kappa.$$
\end{proof}

The value of the optimization problem~\eqref{eq:rucb-opt} can be related to
the expected instantaneous regret of policy drawn randomly from the
distribution $P_t$.
\begin{lem} \label{lem:true-value}
Assume Condition~\ref{cond:rucb}.
Then
$$\sum_{\pi\in\Pi} P_t(\pi) \Delta_D(\pi)
\leq \left(220 + 4\sqrt{2\theta}\right) \cdot \sqrt{\frac{K C_{t-1}}{t-1}} +
2\vepsopt{t}$$
for all $t > t_1$.
\end{lem}
\begin{proof}
Fix any $\pi \in \Pi$ and $t > t_1$.
By the deviation bounds, we have
\begin{align*}
&\Bigl( \eta_D(\pi_{t-1}) - \eta_D(\pi) \Bigr)\\
&\leq \Delta_{t-1}(\pi) + 2\sqrt{\frac{(\vmax{t-1}(\pi) + \vmax{t-1}(\pi_{t-1})) C_{t-1}}{t-1}}\\
&\leq \Delta_{t-1}(\pi) +  2\sqrt{\frac{\left(\vmax{t-1}(\pi) + \theta  K\right) C_{t-1}}{t-1}},
\end{align*}
by Lemma~\ref{lem:vbound}. By
Corollary~\ref{cor:vbound} we have
$$\Delta_D(\pi_{t-1}) \leq 2\sqrt{\frac{2\theta K C_{t-1}}{t-1}}$$
Thus, we get
\begin{align*}
\Delta_D(\pi)
&\leq \Bigl( \eta_D(\pi_{t-1}) - \eta_D(\pi) \Bigr) + \Delta_D(\pi_{t-1}) \\
& \leq \Delta_{t-1}(\pi)
+ 2\sqrt{\frac{\left(\vmax{t-1}(\pi) + \theta  K\right) C_{t-1}}{t-1}}
\\
& \quad{}
+ 2\sqrt{\frac{2\theta K C_{t-1}}{t-1}}
.
\end{align*}
If $\vmax{t-1}(\pi) \leq \theta  K$, then we have
\[
\Delta_D(\pi)
\leq \Delta_{t-1}(\pi) + 4\sqrt{\frac{2\theta K C_{t-1}}{t-1}}
.
\]

Otherwise, Lemma~\ref{lem:monotone} implies that
\[ \vmax{t-1}(\pi) \leq \frac{(t-1) \cdot \Delta_{t-1}(\pi)^2}{8 C_{t-1}} , \]
so
\begin{align*}
\Delta_D(\pi)
& \leq \Delta_{t-1}(\pi)
+ 2\sqrt{\frac{\Delta_{t-1}(\pi)^2}{8} + \frac{\theta  K C_{t-1}}{t-1}}
\\
&\quad{}
+ 2 \sqrt{\frac{2\theta K C_{t-1}}{t-1}}
\\
& \leq 2\Delta_{t-1}(\pi) + 4\sqrt{\frac{2\theta K C_{t-1}}{t-1}}
.
\end{align*}
Therefore
\begin{align*}
\lefteqn{
\sum_{\pi\in\Pi} P_t(\pi) \Delta_D(\pi)
}
\\
& \leq 2\sum_{\pi\in\Pi} P_t(\pi)
\Delta_{t-1}(\pi) + 4\sqrt{\frac{2\theta K C_{t-1}}{t-1}}
\\
& \leq 2\left( \OPT_t + \vepsopt{t}
 \right) + 4\sqrt{\frac{2\theta K C_{t-1}}{t-1}}
\end{align*}
where $\OPT_t$ is the value of the optimization
problem~\eqref{eq:rucb-opt}.
The conclusion follows from Lemma~\ref{lem:value}.
\end{proof}

We can now finally prove the main regret bound for RUCB.
\begin{proof}[Proof of Theorem~\ref{thm:rucb-regret}]
The regret through the first $t_1$ rounds is trivially bounded by $t_1$.
In the event that Condition~\ref{cond:rucb} holds, we have for all $t \geq
t_1$,
\begin{align*}
\sum_{a \in A} W_t(a) r_t(a)
& \geq \sum_{a \in A} (1 - K\mu_t) W_{P_t}(x_t, a) r_t(a)
\\
& \geq \sum_{a \in A} W_{P_t}(x_t, a) r_t(a) - K\mu_t
\\
& = \sum_{\pi \in \Pi} P_t(\pi) r_t(\pi(x_t)) - K\mu_t
,
\end{align*}
and therefore
\begin{align*}
\lefteqn{
\E_{\substack{(x_t,\vec{r}(t)) \sim D \\ a_t \sim W_t'}}\left[
r_t(a_t)
\right]
} \\
& =
\E_{(x_t,\vec{r}(t)) \sim D}\left[
\sum_{a \in A} W_t'(a) r_t(a)
\right]
\\
& \geq \sum_{\pi \in \Pi} P_t(\pi) \eta_D(\pi) - K\mu_t
\\
& \geq \eta_D(\pimax) - O\left( \sqrt{\frac{K C_{t-1}}{t-1}} +
\vepsopt{t} \right)
\end{align*}
where the last inequality follows from
Lemma~\ref{lem:true-value}.
Summing the bound from $t=t_1+1,\dotsc,T$ gives
\begin{multline*}
\sum_{t=1}^T \E_{\substack{(x_t,\vec{r}(t)) \sim D \\ a_t \sim W_t'}}\left[
\eta_D(\pimax) - r_t(a_t)
\right]
\\
\leq t_1 + O\left( \sqrt{TK \log\left(NT/\delta\right)} \right)
.
\end{multline*}
By Azuma's inequality, the probability that $\sum_{t=1}^T r_t(a_t)$
deviates from its mean by more than $O(\sqrt{T\log(1/\delta)})$ is at most
$\delta$.
Finally, the probability that Condition~\ref{cond:rucb} does not hold is at
most $2\delta$ by Lemma~\ref{lem:deviation-bound},
Theorem~\ref{thm:unlabeled}, and a union bound.
The conclusion follows by a final union bound.
\end{proof}

\section{Details of Oracle-based Algorithm}
\label{sec:oracle-alg-details}

We show how to (approximately) solve $\mA$ using the ellipsoid algorithm with
$\AMO$. Fix a time period $t$. To avoid clutter, (only) in this section we drop the subscript $t-1$ from $\eta_{t-1}(\cdot)$, $\Delta_{t-1}(\cdot)$, and $\history_{t-1}$ so that they becomes $\eta(\cdot)$, $\Delta(\cdot)$, and $\history$ respectively.

In order to use the ellipsoid algorithm, we need to relax the program a little bit in order to ensure that the feasible region has a non-negligible
volume. To do this, we need to obtain some perturbation bounds for the
constraints of $\mA$. The following lemma gives such bounds. For any $\delta > 0$, we define $\CH_\delta$ to be the set of all points within
a distance of $\delta$ from $\CH$.
\begin{lem} \label{lem:sensitivity}
Let $\delta \leq b/4$ be a parameter. Let $U, W \in \CH_{2\delta}$ be points
such that $\|U - W\| \leq \delta$. Then we have
\begin{align}
&|\Delta(U) - \Delta(W)|\ \leq\ \gamma \label{eq:regret-sensitivity}\\
\forall Z &\in \CH_{1}: \notag\\
&\left|\E_{x \sim \history}\left[\sum_a \frac{Z(x, a)}{U'(x,a)}\right] - \E_{x \sim \history}\left[\sum_a \frac{Z(x, a)}{W'(x,a)}\right]\right| \leq \epsilon \label{eq:var-sensitivity}
\end{align}
where $\epsilon = \frac{8\delta}{\mu_t^2}$ and $\gamma = \frac{\delta
}{\mu_t}$.
\end{lem}
\begin{proof}
First, we have
\begin{align*}
|\eta(U) - \eta(W)| &\leq \frac{1}{t-1}\!\!\sum_{(x, a, r, q) \in h} \frac{r}{p}|U(x, a) - W(x, a)|\\
&\leq\ \frac{\delta}{\mu_t} = \gamma,
\end{align*}
which implies (\ref{eq:regret-sensitivity}).

Next, for any $Z \in \CH_{1}$, we have
\begin{align*}
&\left|\sum_a \frac{Z(x, a)}{U'(x,a)} - \sum_a \frac{Z(x, a)}{W'(x,a)}\right|\\ &\leq\ \sum_a |Z(x, a)|\frac{|U'(x, a) - W'(x, a)|}{U'(x, a)W'(x, a)}\\
&\leq\ \frac{8\delta}{\mu_t^2}\ =\ \epsilon.
\end{align*}
In the last inequality, we use the Cauchy-Schwarz inequality, and use the
following facts (here, $Z(x, \cdot)$ denotes the vector $\langle Z(x, a)
\rangle_a$, etc.):
\begin{enumerate*}
\item $\|Z(x, \cdot)\| \leq 2$ since $ Z \in \CH_{1}$,
\item $\|U'(x, \cdot) - W'(x, \cdot)\| \leq \|U(x, \cdot) - W(x, \cdot)\| \leq \delta$, and
\item $U'(x, a) \geq (1-bK)\cdot(-2\delta) + b \geq b/2$,
    for $\delta \leq b/4$, and similarly $W'(x, a) \geq b/2$.
\end{enumerate*}
This implies (\ref{eq:var-sensitivity}).
\end{proof}

We now consider the following relaxed form of $\mA$. Here, $\delta \in (0,
b/4)$ is a parameter. We want to find a point $W \in \mathbb{R}^{(t-1)K}$ such that
\begin{align}
\Delta(W)\ &\leq\ s + \gamma\label{eq:low-regret-relax}\\
W\ & \in\ \CH_\delta\label{eq:ch-gap-relax}\\
\forall Z &\in \CH_{2\delta}: \notag\\
\E_{x \sim \history}&\left[\sum_a \frac{Z(x, a)}{W'(x,a)}\right] & \leq \max\{4K, \beta_t\Delta(Z)^2\} + \epsilon,\label{eq:var-bound-gap-relax}
\end{align}
where $\epsilon$ and $\gamma$ are as defined in Lemma~\ref{lem:sensitivity}.
Call this relaxed program $\mA'$.

We apply the ellipsoid method to $\mA'$ rather than $\mA$. Recall the requirements of Lemma~\ref{lem:ellipsoid-description}: we need an enclosing ball of bounded radius for the feasible region, and the radius of an enclosed ball in the feasible region. The following lemma gives this.
\begin{lem}
The feasible region for $\mA'$ is contained in $B(0, \sqrt{t} + \delta)$, and if $\mA$ is feasible, then it contains a ball of radius $\delta$.
\end{lem}
\begin{proof}
Note that for any $W \in \CH_\delta$, we have $\|W\| \leq \sqrt{t} + \delta$,
so the feasible region lies in $B(0, \sqrt{t} + \delta)$.

Next, if $\mA$ is feasible, let $W^\star \in \CH$ be any feasible solution to
$\mA$. Consider the ball $B(W^\star, \delta)$. Let $U$ be any point in
$B(W^\star, \delta)$. Clearly $U \in \CH_\delta$. By
Lemma~\ref{lem:sensitivity}, assuming $\delta \leq 1/2$, we have for all $Z \in \CH_{2\delta}$,
\begin{align*}
\E_{x \sim \history}\left[\sum_a \frac{Z(x, a)}{U'(x,a)}\right] &\leq \E_{x \sim \history}\left[\sum_a \frac{Z(x, a)}{U'(x,a)}\right] + \epsilon\\
&\leq \max\{4K, \beta_t\Delta(Z)^2\} + \epsilon.
\end{align*}
Also
$$\Delta(U) \leq \Delta(W^\star) + \gamma \leq s + \gamma.$$
Thus, $U$ is feasible for $\mA'$, and hence the entire ball $B(W^\star,
\delta)$ is feasible for $\mA'$.
\end{proof}


We now give the construction of a separation oracle for the feasible region of
$\mA'$ by checking for violations of the constraints. In the following, we use
the word ``iteration'' to indicate one step of either the ellipsoid algorithm
or the perceptron algorithm. Each such iteration involves one call to $\AMO$,
and additional $O(t^2K^2)$ processing time.

Let $W \in \mathbb{R}^{(t-1)K}$ be a candidate point that we want to check for
feasibility for $\mA'$. We can check for violation of the constraint
(\ref{eq:low-regret-relax}) easily, and since it is a linear constraint in $W$,
it automatically yields a separating hyperplane if it is violated.

The harder constraints are (\ref{eq:ch-gap-relax}) and (\ref{eq:var-bound-gap-relax}). Recall that Lemma~\ref{lem:linopt} shows that
that $\AMO$ allows us to do linear optimization over $\CH$ efficiently.
This immediately gives us the following useful corollary:
\begin{cor}
Given a vector $w \in \mathbb{R}^{(t-1)K}$ and $\delta > 0$, we can compute
$\arg\max_{Z \in \CH_\delta} w~\cdot~Z$ using one invocation of $\AMO$.
\end{cor}
\begin{proof}
This follows directly from the following fact:
$$\arg\max_{Z \in \CH_\delta} w \cdot Z\ =\ \frac{\delta}{\|w\|}w + \arg\max_{Z \in \CH} w\cdot Z.$$
\end{proof}

Now we show how to use $\AMO$ to check for constraint (\ref{eq:ch-gap-relax}):
\begin{lem} \label{lem:CH-separation}
Suppose we are given a point $W$. Then in $O(\frac{t}{\delta^2})$ iterations,
if $W \notin \CH_{2\delta}$, we can construct a hyperplane separating $W$ from
$\CH_{\delta}$. Otherwise, we declare correctly that $W \in \CH_{2\delta}$. In
the latter case, we can find an explicit distribution $P$ over policies in
$\Pi$ such that $W_P$ satisfies $\|W_P - W\| \leq 2\delta$.
\end{lem}
\begin{proof}
We run the perceptron algorithm with the origin at $W$ and all points in
$\CH_{\delta}$ being positive examples. The goal of the perceptron algorithm
then is to find a hyperplane going through $W$ that puts all of $\CH_{\delta}$
(strictly) on one side. In each iteration of the perceptron algorithm, we have
a weight vector $w$ that is the normal to a candidate hyperplane, and we need
to find a point $Z \in \CH_{\delta}$ such that $w \cdot (Z - W) \leq 0$ (note
that we have shifted the origin to $W$). To do this, we use $\AMO$ as in
Lemma~\ref{lem:linopt} to find $Z^\star = \arg\max_{Z \in \CH_{\delta}} -w
\cdot Z$. If $w \cdot (Z^\star - W) \leq 0$, we use $Z^\star$ to update $w$
using the perceptron update rule, $w \leftarrow w + (Z^\star - W)$. Otherwise,
we have $w \cdot (Z - W) > 0$ for all $W \in \CH_{\delta}$, and hence we have
found our separating hyperplane.

Now suppose that $W \notin \CH_{2\delta}$, i.e. the distance of $W$ from
$\CH_{\delta}$ is more than $\delta$. Since $\|Z - W\| \leq 2\sqrt{t} + 3\delta
= O(\sqrt{t})$ for all $W \in \CH_{\delta}$ (assuming $\delta = O(\sqrt{t})$),
the perceptron convergence guarantee implies that in $O(\frac{t}{\delta^2})$
iterations we find a separating hyperplane.

If in $k = O(\frac{t}{\delta^2})$ iterations we haven't found a separating
hyperplane, then $W \in \CH_{2\delta}$. In fact the perceptron algorithm gives
a stronger guarantee: if the $k$ policies found in the run of the perceptron
algorithm are $\pi_1, \pi_2, \ldots, \pi_k \in \Pi$, then $W$ is within a
distance of $2\delta$ from their convex hull, $\CH' = \text{conv}(\pi_1, \pi_2,
\ldots, \pi_k)$. This is because a run of the perceptron algorithm on
$\CH'_{2\delta}$ would be identical to that on $\CH_{2\delta}$ for $k$ steps.
We can then compute the explicit distribution over policies $P$ by computing
the Euclidean projection of $W$ on $\CH'$ in $\text{poly}(k)$ time using a
convex quadratic program:
\begin{align*}
\min\ \|W - &\textstyle{\sum}_{i=1}^k P_i\pi_i\|^2\\
\sum_i P_i\ &=\ 1\\
\forall i:\ P_i\ &\geq\ 0
\end{align*}
Solving this quadratic program, we get a distribution $P$ over the policies
$\{\pi_1, \pi_2, \ldots, \pi_k\}$ such that $\|W_P - W\| \leq 2\delta$.
\end{proof}

Finally, we show how to check constraint (\ref{eq:var-bound-gap-relax}):
\begin{lem} \label{lem:variance-ellipsoid}
Suppose we are given a point $W$. In $O(\frac{t^3K^2}{\delta^2} \cdot
\log(\frac{t}{\delta}))$ iterations, we can either find a point $Z \in
\CH_{2\delta}$ such that
$$\E_{x \sim \history}\left[\sum_a \frac{Z(x, a)}{W'(x,a)}\right] \geq \max\{4K, \beta_t\Delta(Z)^2\} + 2\epsilon,$$
or else
we conclude correctly that for all $Z \in \CH$, we have
$$\E_{x \sim \history}\left[\sum_a \frac{Z(x, a)}{W'(x,a)}\right] \leq \max\{4K,
\beta_t\Delta(Z)^2\} + 3\epsilon.$$
\end{lem}
\begin{proof}
We first rewrite $\eta(W)$ as $\eta(W) = w \cdot \pi$, where $w$ is a vector
defined as
$$w(x, a) = \frac{1}{t-1}\sum_{(x', a', r, p) \in h:\ x' = x, a'
= a} \frac{r}{p}.$$ Thus, $\Delta(Z) = v - w \cdot Z$, where $v = \max_{\pi'}
\eta(\pi') = \max_{\pi'} w\cdot \pi'$ which can be computed by using $\AMO$
once.

Next, using the candidate point $W$, compute the vector $u$ defined as $u(x, a)
= \frac{n_x/t}{W'(x, a)}$, where $n_x$ is the number of times $x$ appears in
$\history$, so that $\E_{x \sim \history}\left[\sum_a \frac{Z(x, a)}{W'(x,a)}\right] = u
\cdot Z$. Now, the problem reduces to finding a point $R \in \CH$ which
violates the constraint
    $$u \cdot Z \leq \max\{4K, \beta_t(w \cdot Z - v)^2\} + 3\epsilon.$$

Define
$$f(Z) = \max\{4K, \beta_t(w \cdot Z - v)^2\} + 3\epsilon - u \cdot Z.$$
Note that $f$ is convex function of $Z$. Checking for violation of the above
constraint is equivalent to solving the following (convex) program:
\begin{align}
f(Z)\ &\leq\ 0 \label{eq:viol-sep}\\
Z\ &\in\ \CH \label{eq:ch-sep}
\end{align}
To do this, we again apply the ellipsoid method, but on the relaxed program
\begin{align}
f(Z)\ &\leq\ \epsilon \label{eq:viol-sep-relax}\\
Z\ &\in\ \CH_\delta \label{eq:ch-sep-relax}
\end{align}
To run the ellipsoid algorithm, we need a separation oracle for the program.
Given a candidate solution $Z$, we run the algorithm of
Lemma~\ref{lem:CH-separation}, and if $Z \notin \CH_{2\delta}$, we construct a
hyperplane separating $Z$ from $\CH_\delta$.

Now suppose we conclude that $Z \in \CH_{2\delta}$. Then we construct a
separation oracle for (\ref{eq:viol-sep}) as follows. If $f(Z) > \epsilon$,
then since $f$ is a convex function of $Z$, we can construct a separating
hyperplane as in Lemma~\ref{lem:sep-convex}.

Now we can run the ellipsoid algorithm with the starting ellipsoid being $B(0,
\sqrt{t})$. If there is a point $Z^\star \in \CH$ such that $f(Z^\star) \leq
0$, then consider the ball $B(Z^\star, \frac{4\delta}{5\sqrt{tK}\beta_t})$. For
any $Y \in B(Z^\star, \frac{4\delta}{5\sqrt{tK}\beta_t})$, we have
$$|(u \cdot Z^\star) - (u \cdot Y)| \leq \|u\|\|Z^\star - Y\| \leq
\frac{\epsilon}{2}$$ since $\|u\| \leq \frac{\sqrt{K}}{\mu_t}$. Also,
\begin{align*}
&\beta_t|(w \cdot Z^\star - v)^2 - (w \cdot Y - v)^2| \\
&= \beta_t | (w \cdot Z^\star - w \cdot Y) (w \cdot Z^\star + w \cdot Y - 2v) | \\
&\leq \beta_t\|w\|\|Z^\star-Y\|(\|w\|(\|Z^\star\| + \|Y\|) + 2|v|) \leq \frac{\epsilon}{2},
\end{align*}
since $\|w\| \leq \frac{1}{\mu_t}$, $\|Z^\star\| \leq \sqrt{t}$, $\|Y\| \leq
\sqrt{t} + \delta \leq 2\sqrt{t}$, and $|v| \leq \|w\|\cdot\sqrt{t} \leq
\frac{\sqrt{t}}{\mu_t}$.

Thus, $f(Y) \leq f(Z^\star) + \epsilon \leq \epsilon$, so the entire ball
$B(Z^\star, \frac{4\delta}{5\sqrt{tK}\beta_t})$ is feasible for the relaxed
program.

By Lemma~\ref{lem:ellipsoid-description}, in $O(t^2K^2 \cdot
\log(\frac{tK}{\delta}))$ iterations of the ellipsoid algorithm, we obtain one
of the following:
\begin{enumerate*}
    \item we either find a point $Z \in \CH_{2\delta}$ such that $f(Z) \leq \epsilon$, i.e.
$$\E_{x \sim \history}\left[\sum_a \frac{Z(x, a)}{W'(x,a)}\right]\ \geq\ \max\{4K, \beta_t\Delta(Z)^2\} + 2\epsilon,$$
    \item or else we conclude that the original convex program~(\ref{eq:viol-sep},\ref{eq:ch-sep}) is infeasible, i.e. for all $Z \in \CH$, we have
        $$\E_{x \sim \history}\left[\sum_a \frac{Z(x, a)}{W'(x,a)}\right]\ \leq\ \max\{4K, \beta_t\Delta(Z)^2\} + 3\epsilon.$$
\end{enumerate*}

The total number of invocations of iterations is bounded by $O(t^2K^2 \cdot
\log(\frac{tK}{\delta})) \cdot O(\frac{t}{\delta^2}) =
O(\frac{t^3K^2}{\delta^2} \cdot \log(\frac{tK}{\delta}))$.
\end{proof}

\begin{lem} \label{lem:variance-separation}
Suppose we are given a point $Z \in \CH_{2\delta}$ such that
$$\E_{x \sim \history}\left[\sum_a \frac{Z(x, a)}{W'(x,a)}\right]\ \geq\ \max\{4K,
\beta_t\Delta(Z)^2\} + 2\epsilon.$$ Then we can construct a hyperplane
separating $W$ from all feasible points for $\mA'$.
\end{lem}
\begin{proof}
For notational convenience, define the function
$$f_{Z}(W) :=\!\! \E_{x \sim \history}\left[\sum_a \frac{Z(x, a)}{W'(x,a)}\right] - \max\{4K, \beta_t\Delta(Z)^2\} - 2\epsilon.$$
Note that it is a convex function of $W$. Note that for any point $U$ that is
feasible for $\mA'$, we have $f_Z(U) \leq -\epsilon$, whereas $f_Z(W) \geq 0$.
Thus, by Lemma~\ref{lem:sep-convex}, we can construct the desired separating
hyperplane.
\end{proof}

We can finally prove Theorem~\ref{thm:ellipsoid}:
\begin{proof}{\bf [Theorem~\ref{thm:ellipsoid}.]}
We run the ellipsoid algorithm starting with the ball $B(0, \sqrt{t} +
\delta)$. At each point, we are given a candidate solution $W$ for program
$\mA'$. We check for violation of constraint~(\ref{eq:low-regret-relax}) first.
If it is violated, the constraint, being linear, gives us a separating
hyperplane. Else, we use Lemma~\ref{lem:CH-separation} to check for violation
of constraint~(\ref{eq:ch-gap-relax}). If $W \notin \CH_{2\delta}$, then we can
construct a separating hyperplane. Else, we use
Lemmas~\ref{lem:variance-ellipsoid} and \ref{lem:variance-separation} to check
for violation of constraint~(\ref{eq:var-bound-gap-relax}). If there is a $Z
\in \CH$ such that $\E_{x \sim \history}\left[\sum_a \frac{Z(x, a)}{W'(x,a)}\right]
\geq \max\{4K, \beta_t\Delta(Z)^2\} + 3\epsilon$, then we can find a
separating hyperplane. Else, we conclude that the current point $W$ satisfies
the following constraints:
\begin{align*}
\Delta(W)\ &\leq\ s + \gamma \\ 
\forall Z \in \CH:\ \!\!\!\! \E_{x \sim \history}\left[\sum_a \frac{Z(x, a)}{W'(x,a)}\right] & \leq \max\{4K, \beta_t\Delta(Z)^2\} + 3\epsilon\\ 
W\ & \in\ \CH_{2\delta} 
\end{align*}

We can then use the perceptron-based algorithm of Lemma~\ref{lem:CH-separation}
to ``round'' $W$ to an explicit distribution $P$ over policies in $\Pi$ such
that $W_P$ satisfies $\|W_P - W\| \leq 2\delta$. Then
Lemma~\ref{lem:sensitivity} implies the stated bounds for $W_P$.

By Lemma~\ref{lem:ellipsoid-description}, in $O(t^2K^2\log(\frac{t}{\delta}))$
iterations of the ellipsoid algorithm, we find the point $W$ satisfying the
constraints given above, or declare correctly that $\mA$ is infeasible. In the
worst case, we might have to run the algorithm of
Lemma~\ref{lem:variance-ellipsoid} in every iteration, leading to an upper
bound of $O(t^2K^2\log(\frac{t}{\delta})) \times O(\frac{t^3K^2}{\delta^2}
\cdot \log(\frac{tK}{\delta})) = O(t^5K^4\log^2(\frac{tK}{\delta}))$ on the
number of iterations.
\end{proof}

\end{document}